\documentclass[final,3p,times]{elsarticle}
\usepackage{lineno,hyperref}
\usepackage{amssymb}
\usepackage{amsmath}
\usepackage{mathrsfs}
\usepackage{bm}
\include{pythonlisting}
\usepackage{setspace}
\usepackage{comment}
\usepackage{mathrsfs}
\usepackage{bm}
\usepackage{setspace}
\usepackage{comment}
\usepackage{comment}
\usepackage{subcaption}
\usepackage{color}
\usepackage{amsthm}
\usepackage{xcolor}
\newtheorem{theorem}{Theorem}[section]

\newtheorem{lemma}[theorem]{Lemma}

\newtheorem{assumption}[theorem]{Assumption}
\usepackage{lipsum}
\usepackage{amsfonts}
\usepackage{graphicx}
\usepackage{epstopdf}
\usepackage{algorithmic}

\usepackage{graphicx}
\usepackage{capt-of}
\modulolinenumbers[100]

\journal{\textbf{Neurocomputing (2021), doi: https://doi.org/10.1016/j.neucom.2021.10.036} }

\begin{document}

\begin{frontmatter}

\title{Deep Kronecker neural networks: A general framework for neural networks with adaptive activation functions}


\author{Ameya D. Jagtap$^{1,*}$, Yeonjong Shin$^1$, Kenji Kawaguchi$^2$, George Em Karniadakis$^{1,3}$}
\cortext[mycorrespondingauthor]{Corresponding author Email:   ameyadjagtap@gmail.com, ameya$\_$jagtap@brown.edu (A.D.Jagtap), yeonjong$\_$shin@brown.edu (Y. Shin), kkawaguchi@fas.harvard.edu (K. Kawaguchi),  george$\_$karniadakis@brown.edu (G.E.Karniadakis), \\ 
}

\address{$^1$ Division of Applied Mathematics, Brown University, 182 George Street, Providence, RI 02912, USA}
\address{$^2$ Center of Mathematical Sciences and Applications, Harvard University, Cambridge, MA, 02138, USA.}
\address{$^3$ School of Engineering, Brown University, Providence, RI 02912, USA.}

\begin{abstract}
We propose
a new type of neural networks, Kronecker neural networks (KNNs), that form a general framework for neural networks with
adaptive activation functions. KNNs employ the Kronecker product, which provides an efficient way of constructing a very wide network while keeping the number of parameters  low. Our theoretical analysis reveals that under suitable conditions, KNNs induce a faster decay of the loss than that by the feed-forward networks. This is also empirically verified through a set of computational examples. Furthermore, under certain technical assumptions, we establish global convergence of gradient descent for KNNs. As a specific case, we propose the \textit{Rowdy} activation function that is designed to get rid of any saturation region by injecting sinusoidal fluctuations, which include trainable parameters. The proposed Rowdy activation function can be employed in any neural network architecture like feed-forward neural networks, Recurrent neural networks, Convolutional neural networks etc. The effectiveness of KNNs with Rowdy activation is demonstrated through 
various computational experiments including function approximation using feed-forward neural networks, solution inference of partial differential equations using the physics-informed neural networks, and standard deep learning benchmark problems using convolutional and fully-connected neural networks. 
\end{abstract}


\end{frontmatter}

\linenumbers

\section{Introduction}
Neural networks have been very effective in diverse applications of machine learning and scientific machine learning \cite{lecun2015deep}. Undoubtedly, how to design neural networks plays a central role in efficient training \cite{leijnen2020neural}. It has been widely known that some network architectures can be trained well and also be generalized well \cite{he2016deep}. In training neural networks, there are many known open issues, such as the vanishing and exploding gradient and the plateau phenomenon \cite{lu2019dying,ainsworth2020plateau,hanin2018neural}. There are some theoretical works claiming that over-parameterized neural networks
trained by gradient descent can achieve a zero training loss \cite{allen2018convergence,du2018gradient-shallow,du2018gradient-DNN,oymak2019towards,zou2018stochastic,Jacot_18NTK}. However, in practice, the possible training time is always limited and one needs to leverage between the size of neural networks and the number of epochs of gradient-based optimization. 

It has been empirically found that a well-chosen activation function can help gradient descent to not only converge fast 
but also to generalize well \cite{goodfellow2016deep}. 
A representative example is the rectified linear unit (ReLU) activation that achieves state-of-the-art performance in many image classification problems \cite{krizhevsky2012imagenet}, and it has been one of the most popular activation functions for image classification problems. However, there is no rule of thumb of choosing an optimal activation function. 
This has motivated the use of adaptive activation functions by our group and others, see \cite{agostinelli2014learning,hou2017convnets,jagtap2019adaptive,jagtap2019adaptive,jagtap2020locally,zamora2019adaptive,goyal2019learning,sitzmann2020implicit}, with varying results demonstrating superior performance over non-adaptive fixed activation functions
in various learning tasks.

In the present work, we propose a new type of neural networks, the Kronecker neural networks (KNN),
that utilizes the Kronecker product \cite{graham2018kronecker} in the construction of the weight matrices.
We show that KNN provides a general framework for neural networks with adaptive activation functions,
and many existing ones become special instances of KNNs.
As a matter of fact, the KNN is equivalent to the standard feed-forward neural networks (FNN) with a general adaptive activation function of the following form:
\begin{equation}
    \phi_{\alpha,\omega}(x) = \sum_{k=1}^K \alpha_k \phi_k(\omega_k x), \qquad 
    K \in \mathbb{N}_{\ge 1}, \quad \alpha=(\alpha_k), \quad \omega=(\omega_k),
\end{equation}
where $\phi_k$'s are fixed activation functions and
$\alpha, \omega$ are parameters that could be either trainable or fixed.
Hence, the implementation of the KNN does not require that the Kronecker product actually be computed.
However, the Kronecker product allows one to construct a much wider network than a FNN,
while maintaining almost the same number of parameters.

The main findings of our work are summarized below:
\begin{itemize}
    \item By analyzing the gradient flow dynamics of two-layer networks, 
          we prove theoretically that at least in the beginning of training, 
          the loss by KNNs is strictly smaller than the loss by the FNNs. 
    \item We establish global convergence of gradient flow dynamics for the two-layer KNNs under certain technical conditions. 
    \item We propose the adaptive Rowdy activation functions, which is a particular case of a more general KNN framework. In this case, we choose $\{\phi_1\}$ to be any standard activation function such as ReLU, tanh, ELU, sine, Swish, Softplus, etc., and the remaining $\{\phi_k\}_{k=2}^K$ activation functions are chosen as sinusoidal harmonic functions. The purpose of choosing such sinusoidal functions is to inject bounded but highly non-monotonic, noisy effects to remove the saturation regions from the output of each layer in the network, thereby allows the network to explore more and learn faster. 
\end{itemize}

One of the main weaknesses of deep as well as physics-informed neural networks \cite{raissi2019physics} is related to the problem of spectral bias \cite{rahaman2019spectral, cao2019towards}, which prevents them from learning the high-frequency components of the approximated functions. To overcome this problem a few approaches have been proposed in the literature. In \cite{wang2020multi, liu2020multi} the authors introduced appropriate input scaling factors to convert the problem of
approximating high frequency components to lower frequencies. Tancik et al. \cite{tancik2020fourier} introduced Fourier features networks that can learn high-frequency functions by use of Fourier feature mapping. More recently, Wang et al. \cite{wang2020eigenvector} proposed novel architectures that employ spatio-temporal and multi-scale random Fourier features to learn high-frequencies involved in the target functions. With the proposed Rowdy activation functions, the  high-frequency components in the target function can be captured by introducing the high frequency sinusoidal fluctuations in the activation functions. Moreover, the Rowdy activations can be implemented easily in any neural network architecture such as feed forward neural networks, convolutional neural networks, recurrent neural networks and the more recently proposed DeepOnets \cite{lu2021learning}.
To demonstrate the performance of the Rowdy activation functions and to computationally justify our theoretical findings, 
a number of computational examples 
are presented from function approximation, solving partial differential equations,
as well as standard benchmark problems in machine learning.
We found that the KNNs are effectively trained by gradient-based optimization methods
and outperform standard FNNs in all the examples we considered here.

The remainder of the paper is organized as follows. 
In Section 2 we present the mathematical setup and propose the Kronecker neural networks. In Section 3  
we present theoretical results, and in Section 4 we report various computational examples for function approximation, inferring the solution of partial differential equations and  standard deep learning benchmark problems. Finally, we conclude in Section 5 with a 
summary.

\section{Mathematical Setup and Kronecker Neural Networks}
A feed-forward neural network of depth $D$ is a function defined through 
a composition of multiple layers consisting of 
an input layer, $D-1$ hidden-layers and an output layer. 
In the $l^{th}$ hidden-layer, $N_l$ number of neurons are present. 
Each hidden-layer receives an output ${z}^{l-1} \in\mathbb{R}^{N_{l-1}}$ from the previous layer, where an affine transformation
\begin{equation}\label{afft}
\mathcal{L}_l ({z}^{l-1}) \triangleq {W}^l {z}^{l-1} + {b}^l
\end{equation}
is performed.
Here, $W^l \in\mathbb{R}^{N_l \times N_{l-1}}$ is the weight matrix and $b^l \in\mathbb{R}^{N_{l}}$ is the bias vector associated with the $l^{th}$ layer.
A nonlinear activation function $\phi_1(\cdot)$ is applied to each component of the transformed vector before sending it as an input to the next layer. The activation function is an identity function after an output layer. Thus, the final neural network representation is given by 
\begin{equation} \label{def:FNNN}
    u^{\text{FF}}({z}) = (\mathcal{L}_D \circ \phi_1 \circ \mathcal{L}_{D-1} \circ \ldots \circ \phi_1 \circ \mathcal{L}_1 )({z}),
\end{equation}
where the operator $\circ$ is the composition operator.
Let $\Theta_{FF} = \{{W}^l, {b}^l\}_{l=1}^D$, which represents the trainable parameters in the network.

For a vector $v = [v_1,\cdots, v_n]^T \in \mathbb{R}^n$,
let us recall the various norms of $v$:
\begin{equation*}
    \|v\|_1 = \sum_{i=1}^n |v_i|, \qquad
    \|v\|^2 = \sum_{i=1}^n v_i^2, \qquad
    \|v\|_{\infty} = \max_{1\le i \le n} |v_i|.
\end{equation*}
For a matrix $M \in \mathbb{R}^{m\times n}$ where $m \ge n$,
let $\sigma_{\min}(M)$ be the $n$-th largest singular value of $M$.
Also, the spectral norm and the Frobenius norm are defined as 
\begin{equation*}
    \|M\| = \max_{\|x\| = 1} \|Mx\|, \qquad \|M\|_F^2 = \sum_{i=1}^m \sum_{j=1}^n M_{ij}^2, 
\end{equation*}
respectively, where $M_{ij}$ is the $(i,j)$-component of $M$.
Let $\bm{1}_{s \times t}$ be the matrix of size $s \times t$
whose entries are all 1s.

\subsection{Kronecker Neural Networks}
Let $K$ be a fixed positive integer.
Given a FNN's parameters
$\Theta_{FF} = \{\mathbf{W}^l, \mathbf{b}^l\}_{l=1}^D$,
let us define 
the $l$-th block weight matrix and block bias vector, respectively, by
\begin{align*}
    \bm{1}_{K \times K} \otimes W^l = \begin{bmatrix}
    W^l & \cdots & W^l \\
    \vdots &  \ddots & \vdots \\
    W^l &\cdots & W^l 
    \end{bmatrix} \in \mathbb{R}^{N_lK \times N_{l-1}K},
    \qquad
    \bm{1}_{K \times 1} \otimes b^l = \begin{bmatrix}
    b^l \\
    \vdots \\
    b^l
    \end{bmatrix} \in \mathbb{R}^{N_l K},
\end{align*}
where $\otimes$ is the Kronecker product.
Let us define a block activation function $\vec{\phi}$ that applies block-wise.
That is, for $z_j \in \mathbb{R}^{n}$ for $1 \le j \le K$, 
let $z = [z_1, \cdots, z_K]^T \in \mathbb{R}^{nK}$
and 
\begin{equation*}
    \vec{\phi}(z) = \begin{bmatrix}
    \phi_1(z_1) \\ \vdots \\ \phi_K(z_K)
    \end{bmatrix},
\end{equation*}
where $\phi_j$'s are activation functions applied element-wise.
We then construct a neural network from the block weight matrices and block bias vectors as follows:
Let $z^0 = z$ be the input and $z^1 = (\bm{1}_{K \times K} \otimes W^1)(\bm{1}_{K \times 1} \otimes z^0) + \bm{1}_{K\times 1} \otimes b^1$.
For $2 \le l < D$,
\begin{align*}
    z^l = (\bm{1}_{K \times K} \otimes W^l)\vec{\phi}(z^{l-1}) + \bm{1}_{K \times 1} \otimes b^l,
\end{align*}
and $z^D = (\bm{1}_{1 \times K} \otimes W^D)\vec{\phi}(z^{D-1}) + b^D$.
Then, $z^D$ is a $D$-layer FNN having $KN_l$ number of neurons at the $l$-th layer, while keeping 
the number of network's parameters the same as $u^{FF}$ in Eq. \eqref{def:FNNN}.

In order to properly scale the block weight matrices and the block bias vectors,
we introduce the scaling parameters $\omega^l, \alpha^l \in \mathbb{R}^K$.
Here $\omega^l$ is a column vector and $\alpha^l$ is a row vector.
The scaled block weight matrices and block bias vectors are
given by
\begin{align*}
    \tilde{W}^l =(\omega^l \otimes \alpha^l) \otimes W^l,
    \qquad
    \tilde{b}^l = \omega^l \otimes b^l, \qquad 1\le l < D,
\end{align*}
and $\tilde{W}^D = (\bm{1}_{1 \times K} \otimes W^D)$
and $\tilde{b}^D = b^D$.
For $1 \le l \le D$, let
\begin{equation*} 
    z^l:=\tilde{\mathcal{L}}_l(z^{l-1}) = \tilde{W}^l z^{l-1} + \tilde{b}^l.
\end{equation*}
We then obtain the representation given by
\begin{equation} \label{def:Kroneck}
    u_{\Theta}^{\mathcal{K}}({z}) = (\tilde{\mathcal{L}}_D \circ \vec{\phi} \circ \tilde{\mathcal{L}}_{D-1} \circ \ldots \circ \vec{\phi} \circ \tilde{\mathcal{L}}_1 )(\bm{1}_{K \times 1} \otimes z).
\end{equation}
We refer to this representation as a \textit{Kronecker neural network}.
The set of the network's parameters
is $\Theta_{\mathcal{K}} = \{W^l, b^l\}_{l=1}^D \cup \{\omega^l, \alpha^l\}_{l=1}^{D-1}$.

We note that the number of neurons in each hidden-layers of the Kronecker networks is $K$-times larger than those of the feed-forward (FF) networks.
However, the total number of parameters only differ by $2K(D-1)$ due to the Kronecker product.
Furthermore, the Kronecker network 
can be viewed as a new type of neural networks
that generalize a class of existing feed-forward neural networks,
in particular, 
to utilize adaptive activation functions,
as shown below.

\begin{itemize}
    \item If $K=1$, $\omega^l = \alpha^l = 1$ for all $l$,
    the Kronecker network 
    becomes a standard FF network \eqref{def:FNNN}.
    \item If $K=2$, $\omega^l_1 = 1$, $\omega^l_2 = \omega_2$ for all $l$, 
    $\phi_1(x) = \max\{x, 0\}$,
    and $\phi_2(x) = \max\{-x,0\}$,
    the Kronecker network 
    becomes a FF network with Parametric ReLU activation \cite{he2015delving}. 
    \item If $K=2$, $\omega^l_2 = \omega$ for all $l$, 
    $\phi_1(x) = \max\{x, 0\}$,
    and $\phi_2(x) = (e^x - 1)\cdot \mathbb{I}_{x \le 0}(x)$,
    the Kronecker network 
    becomes a FF network with Exponential Linear Unit (ELU)  activation \cite{clevert2015fast}
    if $\omega^l_1 = 1$ for all $l$, 
    and 
    becomes a FF network with Scaled Exponential Linear Unit (SELU) activation \cite{klambauer2017self} 
    if $\omega^l_1 = \omega'$ for all $l$.
    \item If $K = 1$, the Kronecker network 
    becomes a feed-forward neural network with 
    layer-wise locally adaptive activation functions \cite{jagtap2019adaptive, jagtap2020locally}.
    \item If $\omega^l = 1$ for all $l$ and 
    $\phi_k(x) = x^{k-1}$ for all $k$, 
    the Kronecker network 
    becomes a feed-forward neural network with 
    self-learnable activation functions (SLAF) \cite{goyal2019learning}.
    Similarly, a FNN with smooth adaptive activation function \cite{hou2017convnets}
    can be represented by a Kronecker network.
\end{itemize}

The Kronecker network can be efficiently implemented without constructing
the block weight matrices $\{\tilde{W}^l\}_l$ and block bias vectors $\{\tilde{b}^l\}_l$.
It can be checked that 
the Kronecker neural network can be expressed by the composition 
$$ 
u^{\mathcal{K}}({z}) = \left(\mathcal{L}_D \circ \tilde{\phi}^{D-1} \circ \mathcal{L}_{D-1} \circ \ldots \circ \tilde{\phi}^1 \circ \mathcal{L}_1 \right)({z}),
$$
where the activation function at the $l$-th layer is no longer deterministic but depends on the trainable parameters $\{\omega^l, \alpha^l\}$
\begin{align*}
    \tilde{\phi}^l(\mathcal{L}_l(z);\omega^l,\alpha^l) = 
    \sum_{k=1}^K \alpha_k^l \phi_k(\omega^l_k \mathcal{L}_l(z)),~~ l = 1,\cdots, D-1.
\end{align*}

Figure \ref{fig:RowdyNN} shows a schematic of a three-layer Kronecker neural network.
\begin{figure} [htpb] 
\centering
\includegraphics[trim=6cm 7cm 10.6cm 3.5cm, clip=true, scale=0.75, angle = 0]{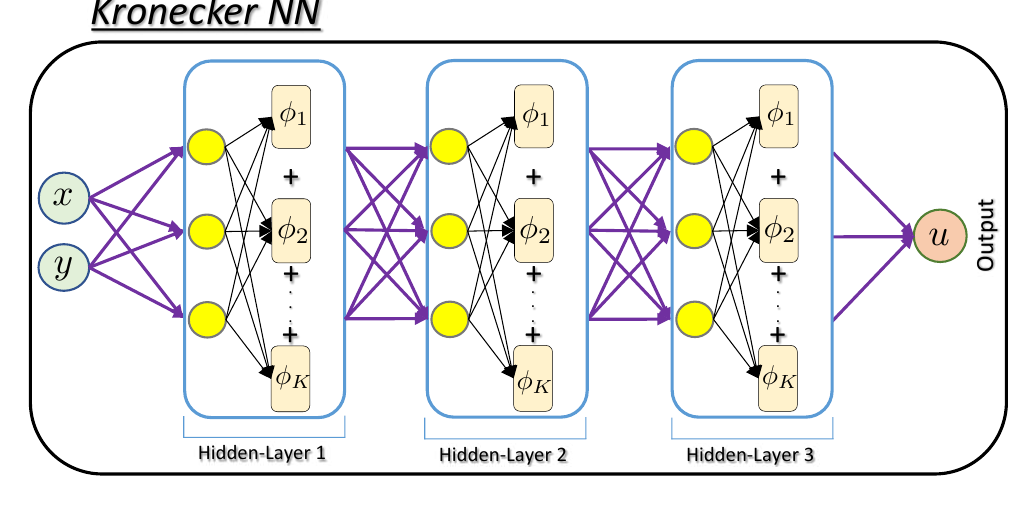}
\caption{Schematic of a three hidden-layers Kronecker neural network. The yellow circles represent the neurons in the respective hidden-layers. Unlike the traditional neural network architecture, the output of the neuron in the KNN architecture passes to more than one activation functions.
}
\label{fig:RowdyNN}
\end{figure}

\section{Gradient Flow Analysis of the Kronecker networks}
In this section, we analyze 
Kronecker networks
in the setup of the supervised learning with the square loss function.
Let $\mathcal{T}_m=\{({x}_i, {y}_i)\}_{i=1}^m$ be a set of $m$-training data points.
The square loss is defined by
\begin{equation}
    L(\Theta) = \frac{1}{2} \sum_{j=1}^m (u^\text{Type}_{\Theta}(x_i) - y_i)^2.
\end{equation}
Here $u^\text{Type}_{\Theta}(x)$ is a selected network that could be either 
a standard FF network $u^\text{FF}_{\Theta}(x)$ \eqref{def:FNNN} or 
a Kronecker network $u^{\mathcal{K}}_{\Theta}(x)$ \eqref{def:Kroneck}.
Specifically, we consider two-layer networks:
\begin{align*}
        u^{\text{FF}}_{\Theta_{FF}}(x) = \sum_{i=1}^N c_i \phi_1(w_i^Tx + b_i),
        \qquad 
        u^{\mathcal{K}}_{\Theta_{\mathcal{K}}}(x) = 
        \sum_{i=1}^N c_i \left[ \sum_{k=1}^K \alpha_k\phi_k(\omega_k(w_i^Tx + b_i))\right].
\end{align*} 
Their corresponding network parameters are denoted by
$\Theta_{FF} = \{c_i, w_i, b_i\}_{i=1}^N$
and $\Theta_{\mathcal{K}} = \{c_i, w_i, b_i\}_{i=1}^N \cup \{\alpha_k, \omega_k\}_{k=1}^K$,
respectively.
The goal of learning is to find the network parameters that minimize the loss function:
\begin{equation} \label{def:problem}
    \min_{\Theta_{Type}} L(\Theta_{Type}) \qquad \text{where} \qquad
    \text{Type} = \text{FF or } \mathcal{K}. 
\end{equation}
The gradient descent algorithm is typically applied to solve the minimization problem \eqref{def:problem}.
The algorithm commences with an initialization of the parameters, $\Theta^{(0)}$.
At the $k$-th iteration, the parameters are updated according to
\begin{align*}
    \Theta^{(k)} &= \Theta^{(k-1)} - \eta_k \nabla_{\Theta} L(\Theta)\big|_{\Theta = \Theta^{(k)}}, 
\end{align*}
where $\eta_k > 0$ is the learning rate for the $k$-th iteration.
The learning rates are typically chosen to be small enough to ensure the convergence.
By letting $\eta_k \to 0$, we obtain the gradient flow dynamics, the continuous version of gradient descent.
The gradient flow dynamics describes the evolution of the network parameters that change continuously in time:
\begin{equation} \label{def:GradFlow}
    \begin{split}
        \dot{\Theta}(t) = -\nabla_{\Theta} L(\Theta(t)),
        \qquad t \ge 0,
        \qquad
        \Theta(0) = \Theta^{(0)}.
    \end{split}
\end{equation}
The loss function, with a slight abuse of notation, is written as $L(t)$.
If the Kronecker network is employed, we write the loss function as
$L^{\mathcal{K}}(t)$;
if the FF network is employed, we write the loss function as 
$L^\text{FF}(t)$.

For the analysis, we make the following assumptions on the parameter initialization and the activation functions.
\begin{assumption} \label{assumption:init}
    Let $c = [c_1,\cdots, c_N]^T$ and $c_i$'s are independently initialized from a continuous probability distribution that is symmetric around 0.
   Also, $v_i = [w_i; b_i]$'s are independently initialized from a normal distribution $N(0,I_{d+1})$,
   where $I_N$ is the identity matrix of size $N$.
   Let $\omega = \bm{1}_{K \times 1}$, where 
    $\bm{1}_{s \times t}$ is the matrix of size $s \times t$
    whose entries are all 1s.
\end{assumption}

\begin{assumption} \label{assumption:activation}
    Let $\phi_k \in C^1(\mathbb{R})$ for all $k=1,\cdots, K$.
    Let $K \ge m$.
    For any distinct $m$ data points $\{z_j\}_{j=1}^m$ in $\mathbb{R}$,
    the $K \times m$ 
    matrix $\bm{\Phi}$
    is full rank,
    where it is defined as 
    \begin{align*}
        [\bm{\Phi}]_{kj} = \phi_k(z_j), \qquad 1\le k \le K, 1 \le j \le m.  
    \end{align*}
\end{assumption}



Suppose that a short period of training time is allowed.
Due to the limited computational resources,
we may frequently encounter such time-limitation scenarios.
Firstly, we are interested in understanding which networks, between the Kronecker and the FF, is more favorable for the training in terms of the loss.
Will there be any advantages of using the Kronecker network over the standard feed-forward?
In what follows, we show that the Kronecker network produces a smaller loss than
that by the standard FF network at least during the early phase of the training.
To fairly compare two networks, we consider 
the following initialization 
that makes $L^{\mathcal{K}}(0) = L^{\text{FF}}(0)$.
For any FF initialization $\Theta_{FF}(0)$,
we initialize the Rowdy network as follows:
${\Theta}_{\mathcal{K}}(0) = {\Theta}_{\text{FF}}(0) \cup 
\{\omega(0), \alpha(0)\}$
where 
\begin{equation} \label{init:1}
    \omega(0) = \bm{1}_{K \times 1}, \qquad
\alpha(0) = \begin{bmatrix}
1 & 0 & \cdots & 0
\end{bmatrix}.
\end{equation}
This makes the two networks at the initialization identical,
which leads to the identical loss value $L^{\mathcal{K}}(0) = L^{\text{FF}}(0)$.

\begin{theorem} \label{thm:small-loss}
    Suppose Assumptions~\ref{assumption:activation} and ~\ref{assumption:init}
    are satisfied.
    Suppose $\alpha$ and $\omega$ are initialized according to \eqref{init:1}.
    Then, with probability $1$ over initialization, 
    there exists $T > 0$ such that 
    \begin{align*}
        L^{\mathcal{K}}(t) < L^{\text{FF}}(t), \qquad \forall t \in (0, T).
    \end{align*}
\end{theorem}
\begin{proof}
    The proof can be found in appendix \ref{app:thm:small-loss}.
\end{proof}

Theorem~\ref{thm:small-loss} shows that the Kronecker network induces a faster decay of the loss than that by the FF network at the beginning of the training.
We remark, however, that this does not imply that the training loss by the Kronecker
will always remain smaller than the loss by the FF.

Next, we show that 
two-layer Kronecker networks whose parameters follow the gradient flow dynamics \eqref{def:GradFlow} can achieve a zero training loss.
For the convergence analysis,
we assume that $\omega$ and $c$ are fixed and we train only
for $\{w_i,b_i\}_{i=1}^N\cup \{\alpha_k\}_{k=1}^K$.
From the training data set $\{(x_i,y_i)\}_{i=1}^m$,
without loss of generality, 
we assume that $\tilde{x}_i = [x_i; 1/\sqrt{2}]$ 
such that $\|\tilde{x}_i\| = 1$ for $1 \le i \le m$.
Let 
\begin{equation}
    X = \begin{bmatrix}
    \tilde{x}_1^T \\ \vdots \\ \tilde{x}_m^T
    \end{bmatrix} \in \mathbb{R}^{m\times (d+1)},
    \qquad
    \bm{y} = \begin{bmatrix}
    y_1 \\ \vdots \\ y_m
    \end{bmatrix} \in \mathbb{R}^{m}.
\end{equation}

\begin{theorem} \label{thm:convergence}
    Under Assumptions~\ref{assumption:init}
    and \ref{assumption:activation},
    suppose $c_i = \frac{\|\bm{y}\|}{Kn\sqrt{m}}\xi_i$ where $\xi_i$'s are 
    independently and identically distributed random variables 
    from the Bernoulli distribution with $p=0.5$,
    $\phi_k(x)$'s are bounded by $B$ in $\mathbb{R}$,
    $\alpha$ is initialized to satisfy
    \begin{equation*}
        K\|\alpha\|_{\infty} \le 1,
        \qquad
        \sum_{k=1}^K \alpha_k \mathbb{E}_{z\sim N(0,1)}[\phi_k(z)] = 0.
    \end{equation*}
    Suppose further that for $\delta \in (0,1)$, 
    $K$ satisfies 
    \begin{equation} \label{cond:K}
        \left(\lambda_0\sqrt{K} -2\|\bm{y}\|^2B\right)K
        \ge 2(1+\delta)\|\bm{y}\|^2B^2,
    \end{equation}
    where $\lambda_0$ is defined in Lemma~\ref{lem:nonsingular-A}.
    Then, with probability at least $1- e^{-\frac{m\delta^2}{2\|X\|^2}}$, we have
    \begin{align*}
        L^{\mathcal{K}}(t) \le L^{\mathcal{K}}(0)e^{-\frac{\lambda_0}{2} t}, \qquad \forall t \ge 0.
    \end{align*}
\end{theorem}
\begin{proof}
    The proof can be found in appendix \ref{app:thm:convergence}.
\end{proof}

In particular, if $B = 1$, it can be checked that 
a sufficient condition for \eqref{cond:K} is 
$K \ge (1+\sqrt{1+4\lambda_0})\frac{\|\bm{y}\|^2}{\lambda_0}$.
Since $\|\bm{y}\|^2 = m\cdot  \frac{1}{m}\sum_{i=1}^m y_i^2$,
we have $K = \mathcal{O}(m)$.
Its corresponding number of parameters of the Kronecker network
is 
$2K + N(d+2) = \mathcal{O}(m) + N(d+2)$.

It is worth mentioning that 
several works \cite{allen2018convergence,du2018gradient-shallow,du2018gradient-DNN,oymak2019towards,zou2018stochastic,Jacot_18NTK}
analyzed over-parameterized neural networks (that is, the number of network parameters is significantly larger than the number of training data) and showed that
gradient descent can train neural networks to interpolate all the training data.
Unlike such existing results on the global convergence of gradient descent for significantly over-parameterized two-layer FF networks,
the two-layer Kronecker NN does not require such severe over-parameterization.
The required number of parameters is merely 
is $(d+2)N + 2K$,
with $K =\mathcal{O}(m)$.

\section{Computational Examples}
\begin{figure} [htpb] 
\centering
\includegraphics[trim=0.4cm 0.8cm 0.4cm 0.8cm, scale=0.55, angle=0]{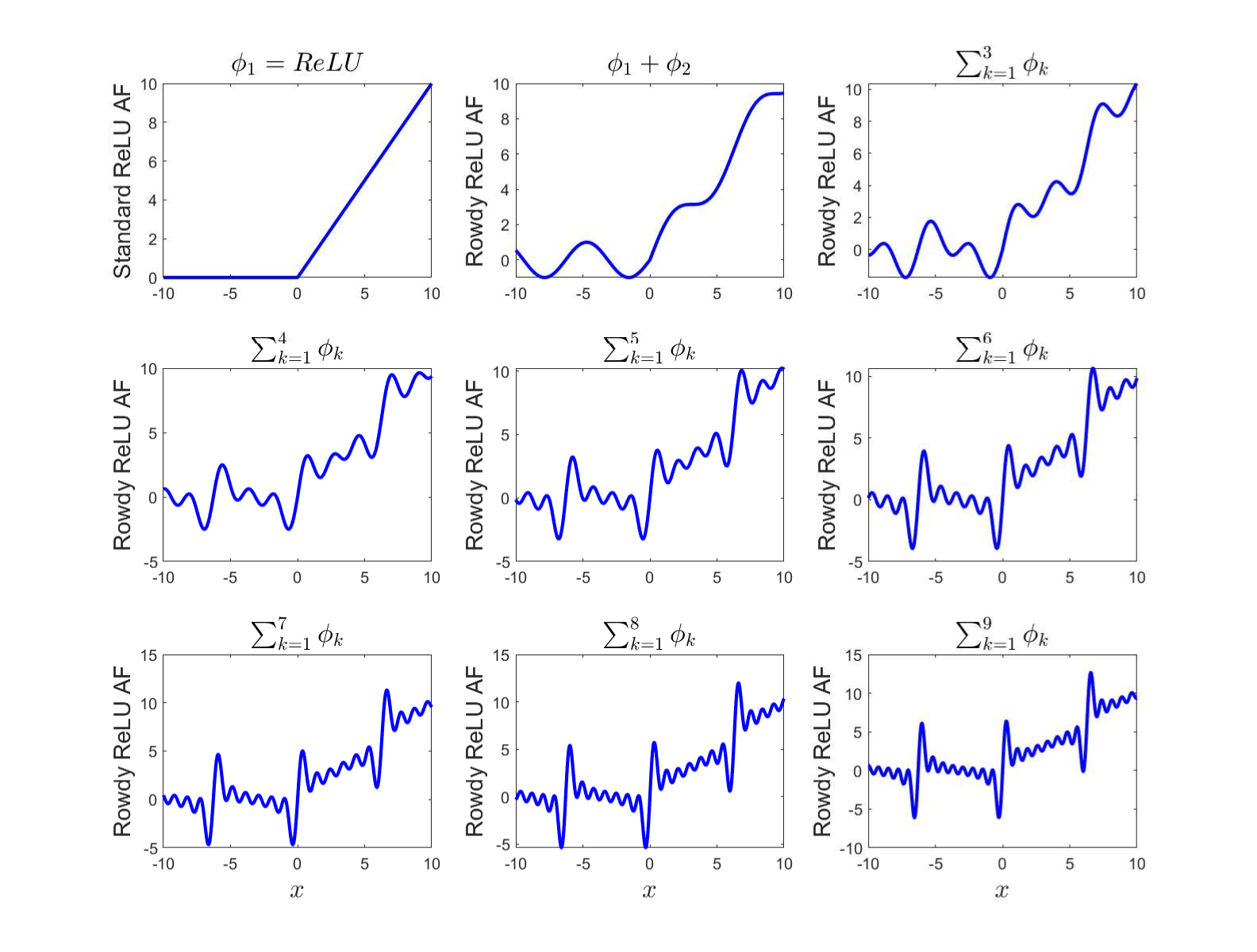}
\caption{Rowdy activation functions: The standard $\phi_1 =$ ReLU activation function, and the remaining $\phi_k = n\sin((k-1)nx), k\geq 2$ activation functions with $n=1$.}\label{fig:RowdyAF}
\end{figure}

The KNN is a general framework for the adaptive activation functions where any combination of activation functions can be chosen. Thus, there is no unique way to choose these activation functions. In this regard, we propose the \textit{Rowdy activation functions} and we refer to the corresponding KNN as \textit{Rowdy-Net} (a neural network with Rowdy activation functions). In the Rowdy network, we choose $\{\phi_1\}$ to be any standard activation function such as ReLU, tanh, ELU, sine, Swish, Softplus etc., and the remaining $\{\phi_k\}_{k=2}^K$ activation functions are chosen as
\begin{equation} \label{def:Rowdy}
     \phi_k(x) = n\sin((k-1)nx) \quad  \text{or} \quad n\cos((k-1)nx), \quad \forall 2 \le k\le K,
\end{equation}
where $n \geq 1$ is the fixed positive number acting as scaling factor.
The word `\textit{Rowdy}' means highly fluctuating/irregular/noisy, which is used to signify the $\{\phi_k\}_{k=2}^K$ activation functions.   Figure \ref{fig:RowdyAF} shows the Rowdy ReLU activation functions with $n=1$ for different $K$ terms. The purpose of choosing such fluctuating terms is to inject bounded but highly non-monotonic, noisy effects to remove the saturation regions from the output of each layer in the network. On similar grounds, Gulcehre et al. \cite{gulcehre2016noisy} proposed noisy activation functions by adding random noise. Similarly, Lee et al. \cite{lee2019probact} proposed probabilistic activation functions for deep neural networks.

The scaling factor $n$ plays an important role in terms of convergence of the network training process. There is no rule of thumb for choosing the value of scaling factor, which basically depends on the specific problem. Our numerical experiments show that for regression problems like function approximation and inferring the solution of PDEs, values of $ n \geq 1$ can accelerate the convergence, but larger values of $n$ can make the optimization algorithm sensitive. The scaling factor defined with the trainable parameters are initialized in such a way that the initial activation of Rowdy-Net is the same as the corresponding standard activation function. 
For more details on the scaling factor, see Jagtap et al. \cite{jagtap2019adaptive, jagtap2020locally}.
In this section, we shall demonstrate the efficiency of the Rowdy-Net by comparing its performance with the fixed (standard), and layer-wise locally adaptive (L-LAAF) for various regression and classification test cases. For convenience, we clearly mention the fixed (f) and the trainable (t) parameters with their respective initialization for all three types of activation functions, namely,
\begin{align*}
\textbf{Fixed AF} & :  \alpha_1^l = 1 ~(\text{f});~~~ \alpha_k^l = 0, \forall k \geq 2 ~(\text{f});~~~ \omega_1^l = 1  ~(\text{f});~~~ \omega_k^l = 0, \forall k \geq 2 ~(\text{f}),
\\ \textbf{L-LAAF} & : \alpha_1^l = 1 ~(\text{f});~~~ \alpha_k^l = 0, \forall k \geq 2 ~(\text{f});~~~ \omega_1^l = 1  ~(\text{t});~~~ \omega_k^l = 0, \forall k \geq 2 ~(\text{f}),
\\ \textbf{Rowdy AF} & : \alpha_1^l = 1 ~(\text{f});~~~ \alpha_k^l = 0, \forall k \geq 2 ~(\text{t});~~~ \omega_1^l = 1  ~(\text{t});~~~ \omega_k^l = 1, \forall k \geq 2  ~(\text{t}),
\end{align*}
for all $l$. Note that with this initialization, the initial activation functions for L-LAAF as well as Rowdy-Net in each hidden-layer are the same as the fixed activation function. Moreover, the multiplication of scaling factor $n$ in \eqref{def:Rowdy}
with any trainable parameter, say, $\omega_1^l$, the initialization is done such that $n\omega_1^l = 1, \forall n$.

It is important to note that the KNNs used in our experiments are the  modifications of the widely used neural network models such as FNNs and CNNs. Accordingly, the plots for the fixed activation function (fixed AF) are the results of these widely used neural network models. For example, Figures \ref{fig:new:1} and \ref{fig:new:2} compare the FNN vs the KNN, as the base architecture is FNN and thus the fixed AF results are those of the FNN. Similarly, Figure \ref{fig:new:3} compares the LeNet (a widely used CNN) and the KNN, and Figure \ref{fig:new:4} compares the ResNet with convolutions (another widely used CNN) and the KNN.

\subsection{Nonlinear discontinuous function approximation}
In this test case we will show the ability of the Rowdy-Net to achieve the machine zero loss value. We also compare the performance of fixed, L-LAAF and Rowdy-Net with different number of terms. We consider the discontinuous nonlinear function given by
$$\mathit{f}(x) = \begin{cases} 0.2~\sin(6x) & \text{For}~x<0 \\ 1+0.1x~\cos(14x) & \text{Otherwise}.\end{cases}
$$
The loss function consists of only the data mismatched term.

\begin{figure} [htpb] 
\centering
\includegraphics[trim=0cm 0.0cm 1cm 0cm, scale=0.47, angle=0]{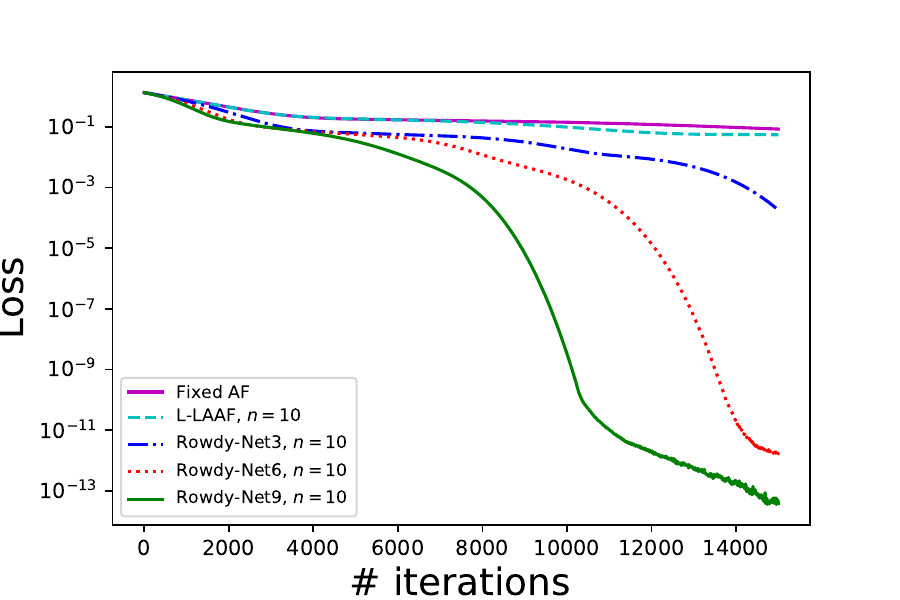}
\includegraphics[trim=0cm 0.0cm 1cm 0cm, scale=0.47, angle=0]{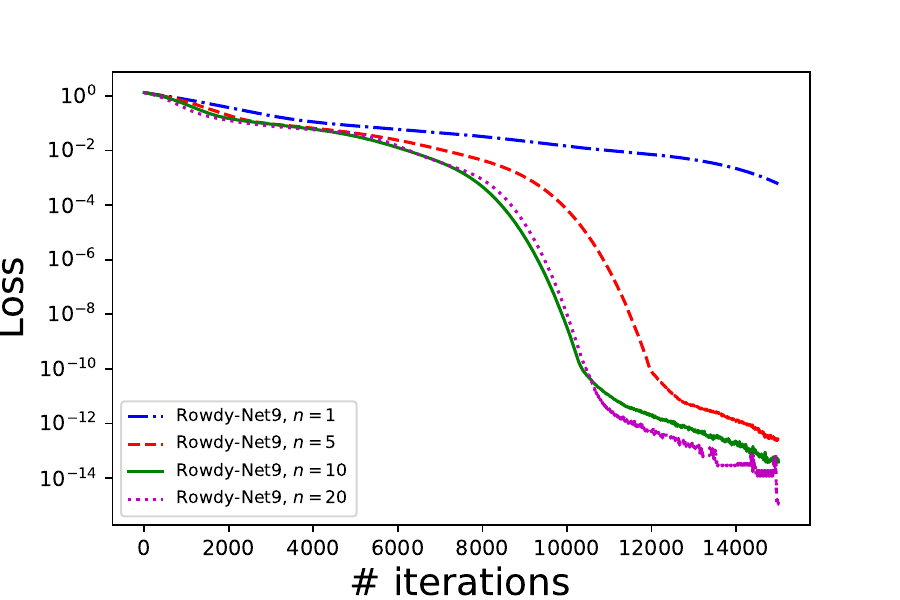}
\caption{Nonlinear discontinuous function approximation: The left figure shows the loss function variation using fixed, L-LAAF and Rowdy activation functions (3, 6 and 9 terms) while the right figure shows the loss function for a 9-term Rowdy network with different scaling factors $n$.}\label{fig:Test0}
\end{figure}
The domain is [-3, 3] and the number of training points is just 5, which are fixed for all cases. We used a single hidden-layer with 40 neurons, cosine activation function, and the learning rate is 8.0e-6. 
Figure \ref{fig:Test0} shows the loss function, and it can be observed that with just  a 2-layer shallow neural network the loss function for Rowdy-Net approaches machine zero precision very quickly. The right figure gives the performance comparison for the fixed, L-LAAF and the Rowdy-Net$K$ (where the number $K$ signifies the first $K$ number of terms used in the computations) for the scaling factor $n = 10$. The Rowdy network performs consistently well by increasing the number of terms. The right figure shows the effect of scaling factor on the performance of the network. It can be seen that  by increasing the scaling factor the network can be trained faster. As discussed before, the initialization  of all adaptive activation functions is done such that they are the same as the fixed activation function at initial step, hence, all loss functions values start from the same point in all the cases. Table \ref{TableTC1} shows the comparison of the total normalized computational cost required for the fixed activation function, L-LAAF and the Rowdy-Net  with 3,6 and 9 terms. 
\begin{figure} [htpb] 
\centering
\includegraphics[trim=0cm 0.0cm 1cm 0cm, scale=0.57, angle=0]{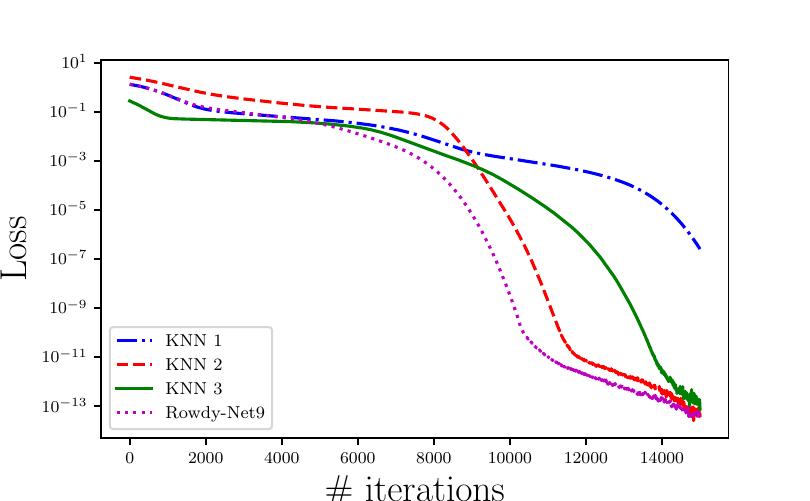}
\caption{Nonlinear discontinuous function approximation: Comparison of different KNN activation functions with Rowdy-Net 9 using scaling factor 10 for all cases.}\label{fig:Test13}
\end{figure}
\begin{table}[htpb]
\begin{center}
\small \begin{tabular}{cccccc} \hline
& Fixed & L-LAAF &  Rowdy-Net3 &Rowdy-Net6 & Rowdy-Net9 \\  \hline 
Normalized time & 1 & 1.12& 1.26 & 1.58 & 1.75 \\
 \hline
 \end{tabular}
\caption{Nonlinear discontinuous function approximation: Comparison of total normalized computation time for fixed activation function, L-LAAF and Rowdy-Net (with 3, 6 and 9 terms) activation functions. The time required for fixed activation function is taken as a baseline.}\label{TableTC1}
\end{center}
\end{table}
In this case,  the time required for fixed activation function is taken as a baseline. By increasing the number of terms in the Rowdy-Net, the computational time increases.
Keeping $\phi_1$ as cosine activation function, we also compare the Rowdy-Net9 with different choices for $\phi_k, k\geq 2$ as defined below:
\begin{align}
\text{KNN}_1 & = \sum_{k=2}^9 \alpha_k^l\tanh(\omega_k^l\mathcal{L}_l(z)), l = 1,\ldots,D-1. \\
\text{KNN}_2 & = \sum_{k=2}^9 \alpha_k^l\text{ReLU}(\omega_k^l\mathcal{L}_l(z)), l = 1,\ldots,D-1.
\end{align}
and in the case of KNN$_3$, the $\phi_k$'s are chosen randomly as $\phi_2 = \tanh, \phi_3 = \text{Sigmoid},\phi_4 = \text{elu},\phi_5 = \text{ReLU}$ $\phi_6 = \tanh,\phi_7 =\tanh,\phi_8 = \text{Softmax},\phi_9 = \text{Swish}$. Figure \ref{fig:Test13} shows the comparison of the loss functions for these activation functions; clearly, Rowdy-Net9 is trained faster.
 
\subsection{High frequency function approximation}
In this test case we consider the high frequency sinusoidal functions given by $\sin(m \pi x),~ m=1,100$ and 200. The domain is [0, 2$\pi$] and the number of training points is 100, which is fixed for all the cases. We use cosine activation function and the learning rate is 4.0e-6. The number of hidden-layers is 3 with 50 neurons in each layer. We use the first nine terms of Rowdy activation functions. The scaling factor $n = 10$ is used in all the cases.

\begin{figure} [htpb] 
\centering
\includegraphics[trim=1.2cm 0cm 0cm 0cm, scale=0.36, angle=0]{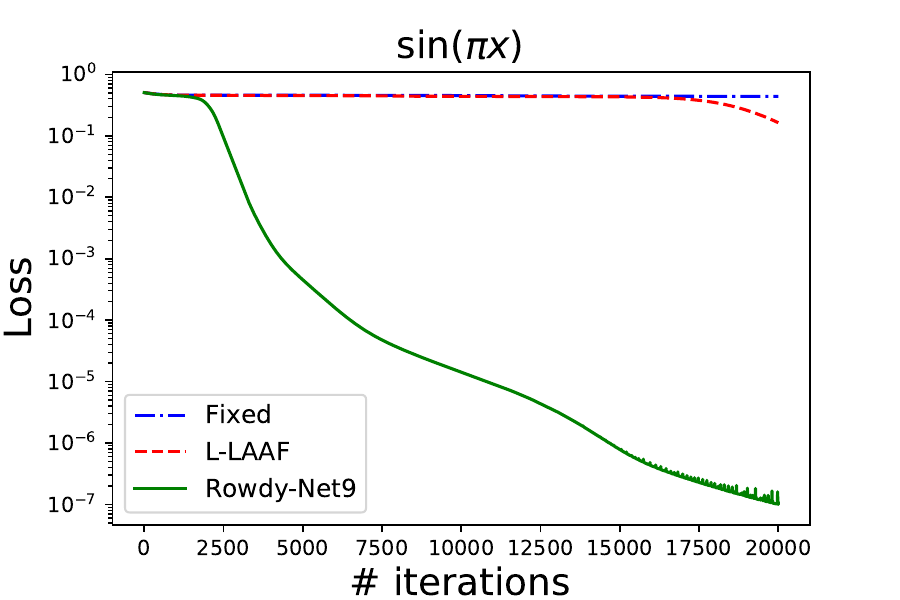}
\includegraphics[trim=1.2cm 0cm 0cm 0cm, scale=0.36, angle=0]{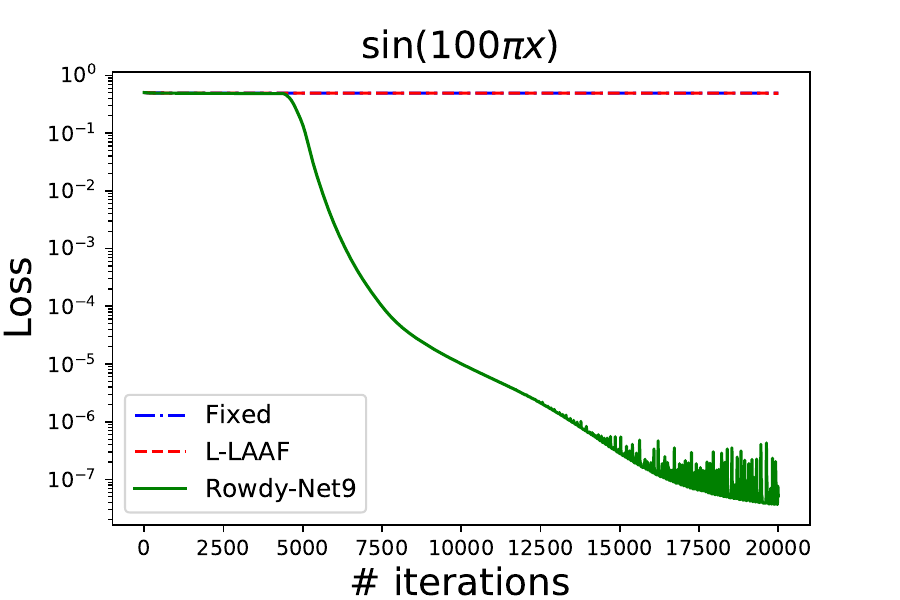}
\includegraphics[trim=1.2cm 0cm 0cm 0cm, scale=0.36, angle=0]{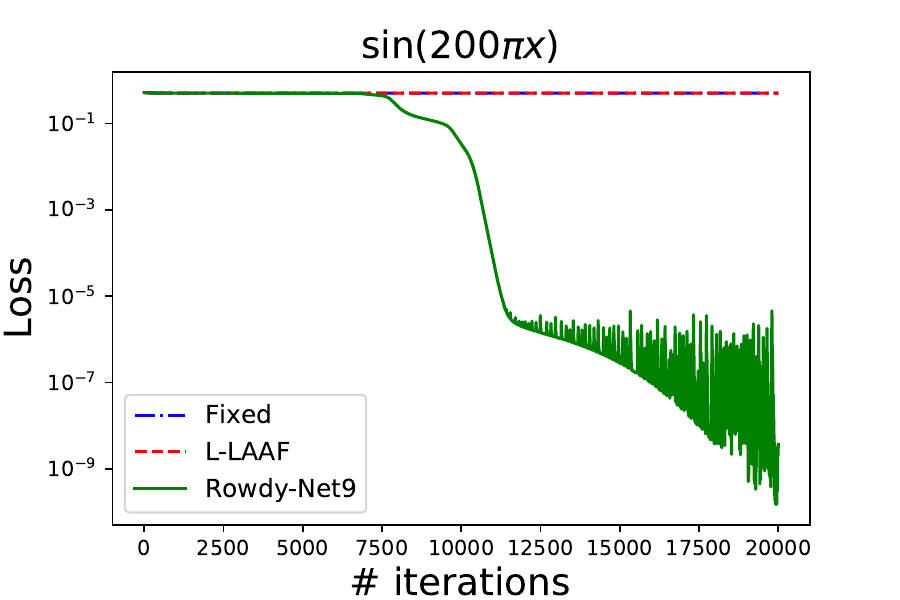}
\caption{High frequency function approximation: Loss function versus number of iterations for $\sin(m\pi x), `m = 1, 100$ and 200, using fixed, L-LAAF and Rowdy-Net9 activation functions.}\label{fig:Test1}
\end{figure}
\begin{table}[htpb]
\begin{center}
\small \begin{tabular}{cccc} \hline
& Fixed & L-LAAF &  Rowdy-Net9 \\  \hline 
Normalized time & 1 & 1.11 & 1.73 \\
 \hline
 \end{tabular}
\caption{High frequency function approximation: Comparison of total normalized computation time for fixed activation function, L-LAAF and Rowdy-Net9 activation function. Time required for fixed activation function is taken as a baseline.}\label{TableTC2}
\end{center}
\end{table}
Figure \ref{fig:Test1} shows the loss functions for fixed, locally adaptive (L-LAAF) and Rowdy-Net9 for $\sin(m\pi x), m = 1, 100$ and 200. In all cases, despite using a small learning rate the Rowdy-Net converges faster than the fixed and locally adaptive activation functions.
Table \ref{TableTC2} shows the comparison of total normalized computational cost required for the fixed, L-LAAF and the Rowdy-Net9 activation functions. Again, the time required for fixed activation function is taken as a baseline. We can see a similar increment in the computation cost requirement for both L-LAAF and Rowdy-Net activation functions compared to fixed activation function.

\subsubsection{Effect of high learning rate}
We again perform the same experiment with higher learning rates (LR) 4.0e-3. Figure \ref{fig:Test2new} shows the $\sin(\pi x)$ function approximation example. Fixed, L-LAAF and Rowdy activation functions converge faster. In the case of both fixed as well as L-LAAF, the loss function goes till 1.0e-6, see figure \ref{fig:Test2new} (left). Furthermore, in the case of Rowdy-Net, it can be seen that the loss function decreases till 1.0e-11, but then suddenly goes up. The main reason is the high learning rate, which can make the parameters in the Rowdy activation function, and in turn, the Rowdy activation function very sensitive during the optimization procedure.
\begin{figure} [htpb] 
\centering
\includegraphics[trim=1.2cm 0cm 0cm 0cm, scale=0.5, angle=0]{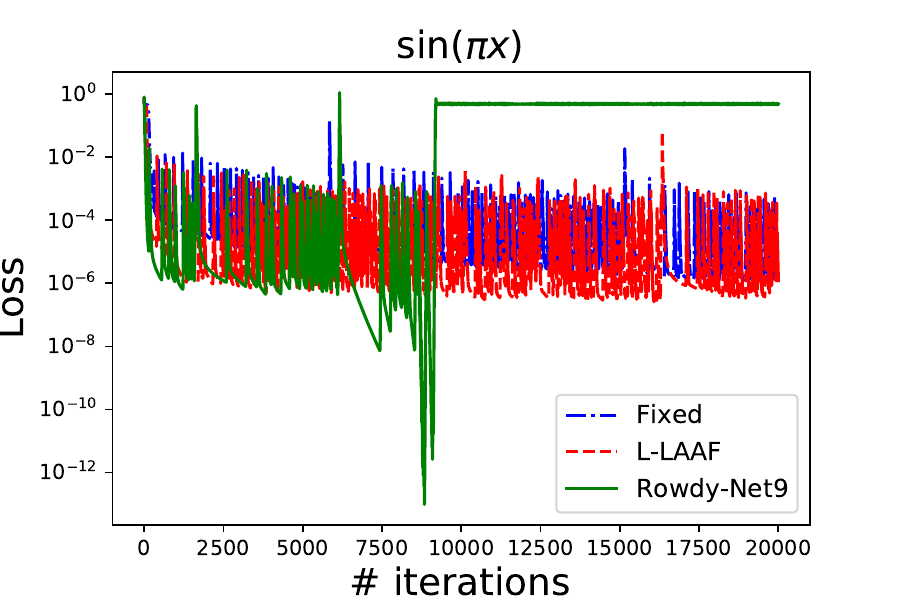}
\includegraphics[trim=1.2cm 0cm 0cm 0cm, scale=0.5, angle=0]{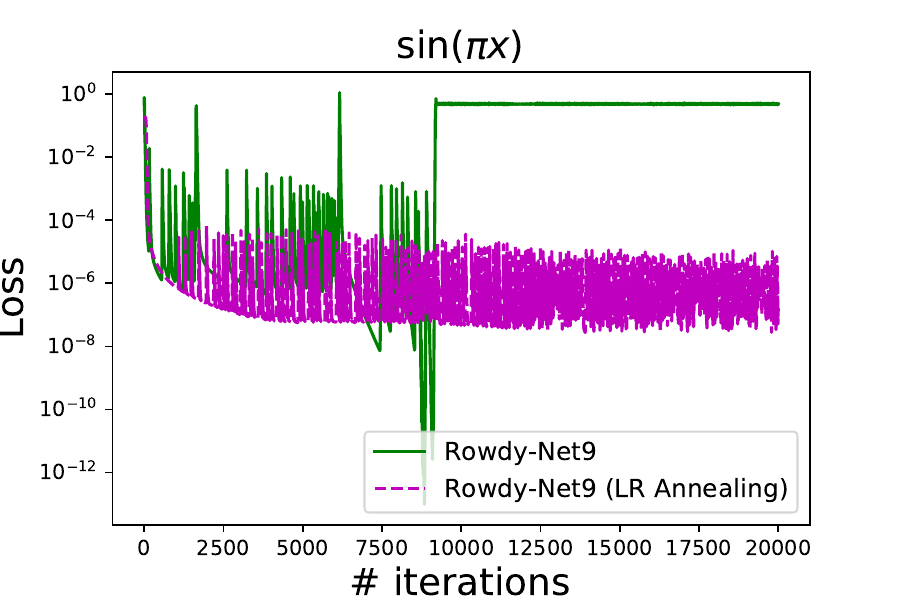}
\caption{Effect of high learning rate on the convergence of fixed, L-LAAF, and Rowdy-Net (left). The Rowdy-Net loss function with and without learning rate annealing (right).}\label{fig:Test2new}
\end{figure}
Such behavior can be avoided by either using the low learning rate or by using the strategy of learning rate annealing \cite{krizhevsky2012imagenet}. The learning rate annealing can significantly affect generalization performance of the neural networks. In particular, training of neural network
with a large initial learning rate followed by a smaller annealed learning rate can outperform the neural network training with the smaller learning rate used throughout.
Figure \ref{fig:Test2new} (right) shows the Rowdy-Net9 results with and without learning rate annealing. In case with learning rate annealing, we decreased the learning rate from 4e-3 to 1e-4 after 500 iterations. The learning rate annealing not only curbs the sensitiveness of Rowdy activation functions, but also decreases the magnitude of oscillations in the loss function.

\subsection{Helmholtz equation}
The Helmholtz equation arises in many real-world problems such as acoustics, vibrating membrane etc. Here we employed Physics-Informed Neural Networks (PINNs) \cite{raissi2019physics} to solve the Helmholtz equation. The PINN is a simple and efficient method for solving partial differential equations involving sparse and noisy data set. The PINN framework can  incorporate the given information like governing equation, experimental as well as synthetic (high resolution numerical solution) data into the loss function, thereby converts the original problem into an optimization problem. The PINN method has been successfully applied to solve many problems in science and engineering, see for examples \cite{jagtap2020extended, kharazmi2021hp, shukla2020physics, mao2020physics, shukla2021parallel, shukla2021physics, cai2021flow, jagtap2020conservative}.

The Helmholtz equation in two dimensions is given by
\begin{equation}\label{2DHel}
 u_{xx} + u_{yy} + k^2 u = g(x,y), \ \ (x,y) \in [-1,~1]^2,
\end{equation}
with appropriate Dirichlet boundary conditions. The forcing term is obtained from the exact solution $u(x,y) = \sin(\pi x)\sin(4\pi y)$ for $k=1$, which is given as
\begin{equation*}
 g(x,y) =  -\pi^2 \sin(\pi x)\sin(4\pi y) - (4\pi)^2 \sin(\pi x)\sin(4\pi y)+ k^2 \sin(\pi x)\sin(4\pi y).
\end{equation*}

\begin{figure} [htpb] 
\centering
\includegraphics[trim=2.9cm 0.0cm 0cm 0cm, scale=0.295, angle = 0]{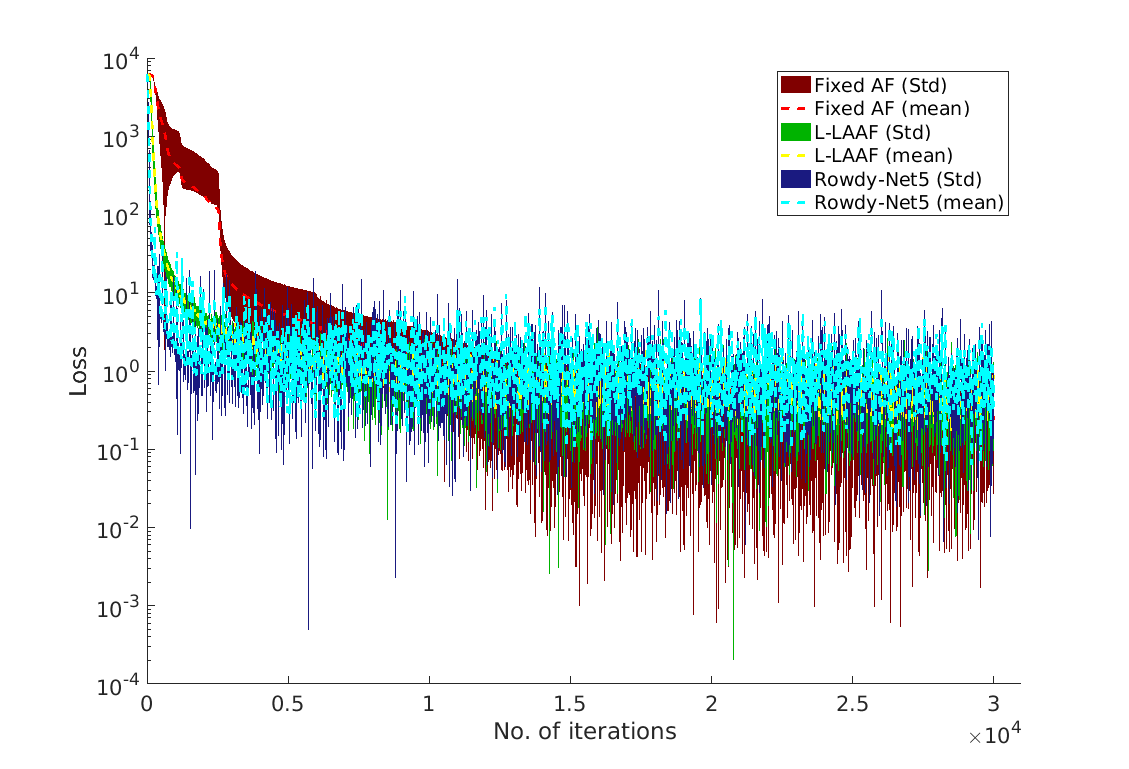}
\includegraphics[trim=2.9cm 0.0cm 1.5cm 0cm, scale=0.295, angle = 0]{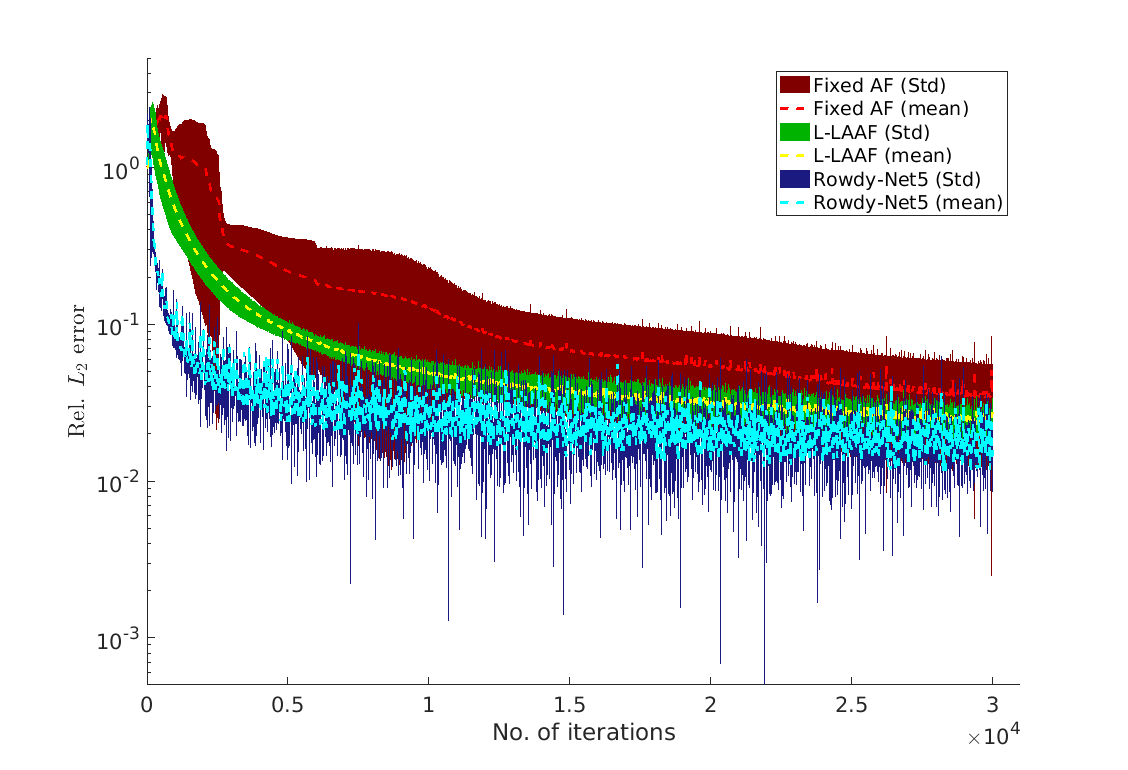}
\caption{Helmholtz equation: Mean and std. deviation of loss function (left) and relative $L_2$ error (right) for up to 30k iterations for fixed AF, L-LAAF and Rowdy-Net5 (5 terms) using 5 different realizations in each case. }\label{fig:rowHe}
\end{figure}

\begin{figure} [htpb] 
\centering
\includegraphics[ scale=0.4, angle = 0]{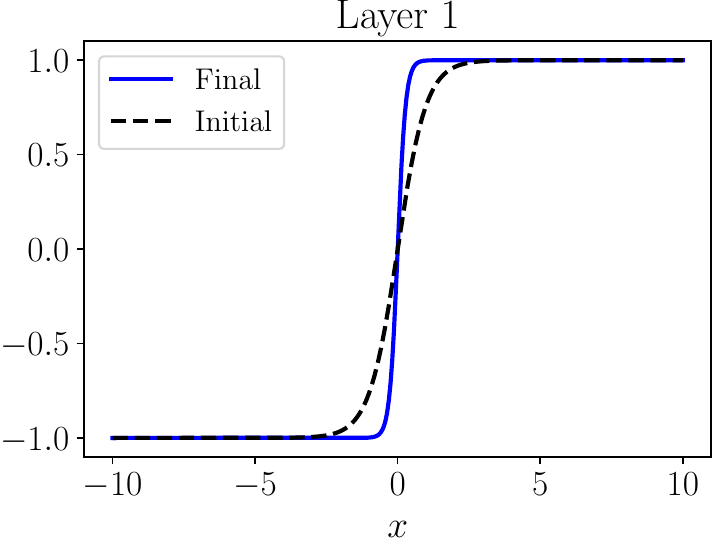}
\includegraphics[ scale=0.4, angle = 0]{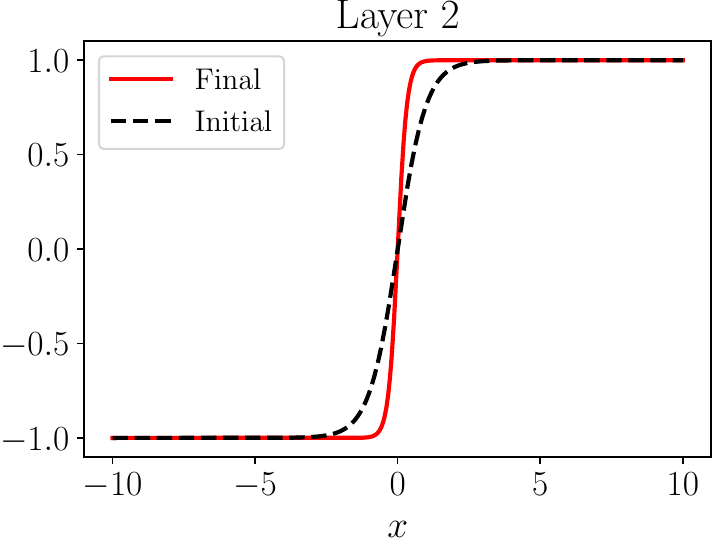}
\includegraphics[ scale=0.4, angle = 0]{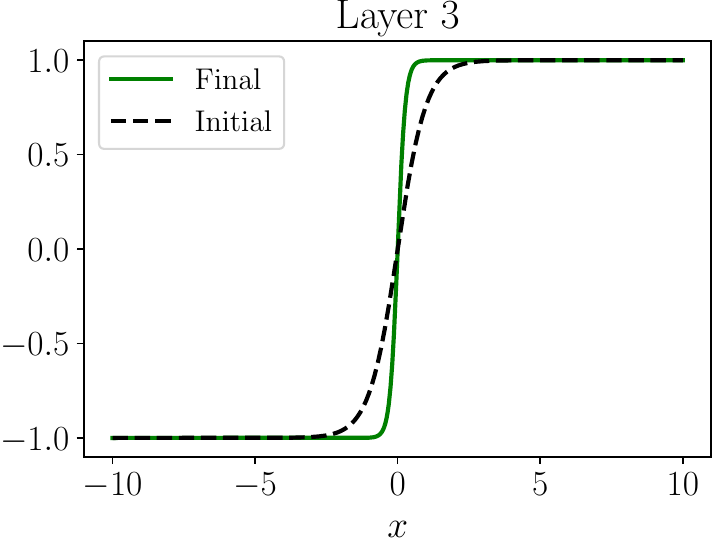}

\includegraphics[ scale=0.41, angle = 0]{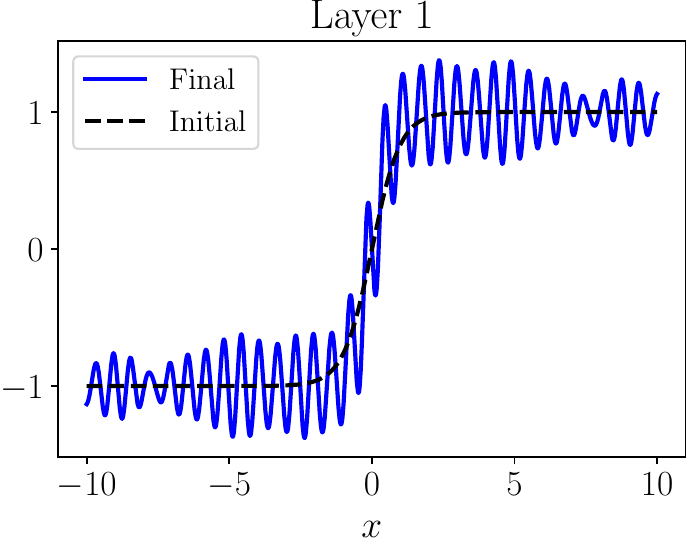}
\includegraphics[ scale=0.41, angle = 0]{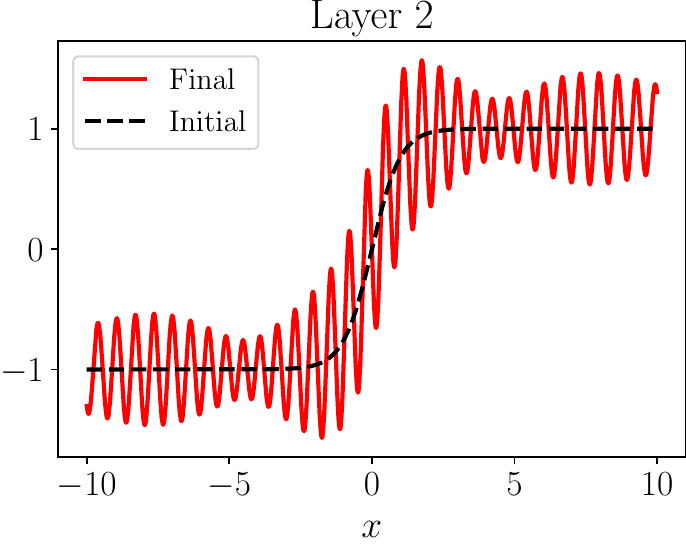}
\includegraphics[ scale=0.41, angle = 0]{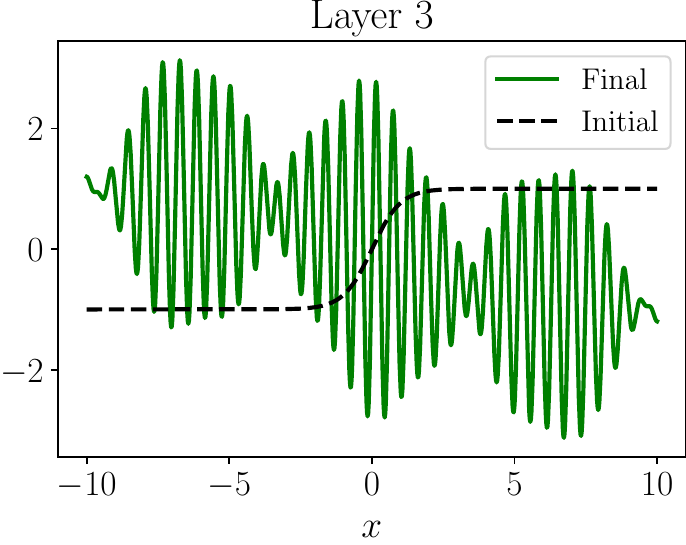}

\caption{Helmholtz equation: Layer-wise L-LAAF (Top row) and Rowdy (bottom row) hyperbolic tangent activation functions (Rowdy-Net5). In the L-LAAF, only the slope of the activation function changes without changing the saturated region but the Rowdy activation function can get rid of the saturated region, hence, it can be trained faster. }\label{fig:rowSig}
\end{figure}
We used a 3 hidden-layers, 30 neurons per layer fully connected neural network with hyperbolic tangent activation function. The number of boundary training points is 300, and the number of residual points is 6000, which are randomly chosen. The learning rate is 8.0e-3 and the optimizer is ADAM.

\begin{figure} [htpb] 
\centering
\includegraphics[trim=1.2cm 0cm 0cm 0cm, scale=0.6, angle=0]{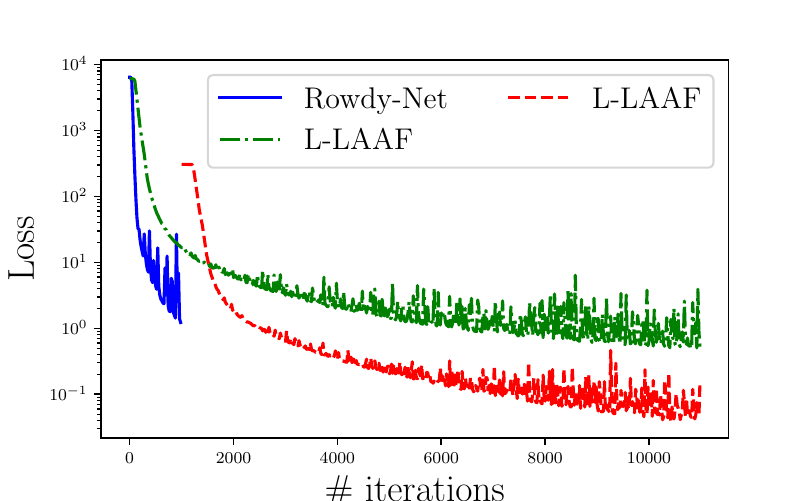}
\includegraphics[trim=1.2cm 0cm 0cm 0cm, scale=0.6, angle=0]{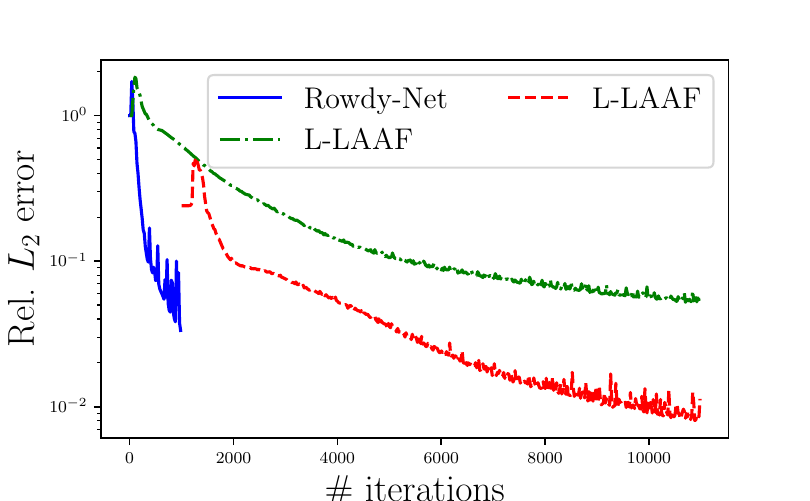}
\caption{Helmholtz equation: Loss function (left) and relative $L_2$ error (right), where the Rowdy network is trained for the first 1000 iterations (blue line) and then switched to L-LAAF network (red dashed line) for the remaining iterations for computational expediency. These results are also compared with L-LAAF as shown by green dash-dot line.}\label{fig:Test3_h}
\end{figure}
For the Helmholtz equation we are using a sine fluctuating part with first 5 (Rowdy-Net5) terms, and we also use the scaling factor $n = 10$ in all cases. Figure \ref{fig:rowHe} shows the loss function (left) and relative $L_2$ error (right) up to 30k iterations for fixed AF, L-LAAF and Rowdy-Net5 using 5 different realizations in each case. It can be seen that the Rowdy-Net performs better than the fixed and locally adaptive activation functions. Figure \ref{fig:rowSig} shows the initial and final L-LAAF (top) and Rowdy-Net5 (bottom) activation function for all three hidden-layers. The initial activation function is the standard activation function in each case. In L-LAAF, only the slope of the activation function increases as expected, but in the case of Rowdy-Net5, the final activation functions are very oscillatory. 
In the case of PINNs, the computational cost increases for Rowdy activation function compared to baseline fixed activation function and L-LAAF. One remedy to reduce the computational cost is to employ the transfer learning strategy, i.e.,  the Rowdy-Net can be trained for the initial period (for few hundred iterations) and then this pre-trained model can be used for L-LAAF network for further training. Figure \ref{fig:Test3_h} shows the loss function and relative $L_2$ errors, where the Rowdy-net was trained for the first 1000 iterations and then switched to L-LAAF network for the remaining iterations. Both the loss and error initially jump after switching from Rowdy to L-LAAF networks but they decay thereafter. These results are also compared with L-LAAF as shown by green dash-dot line. Another way to reduce the computational cost associated with the Rowdy-Net training is, directly include the high-frequency components with the strategy of learning rate annealing.

We further test the convergence speed and accuracy of the proposed Rowdy network for a high frequency solution of Helmholtz equation. In this case the exact solution is assumed to be of the form $u(x,y) = \sin(5\pi x)\sin(10\pi y).$
\begin{figure} [htpb] 
\centering
\includegraphics[trim=1.2cm 0cm 0cm 0cm, scale=0.54, angle=0]{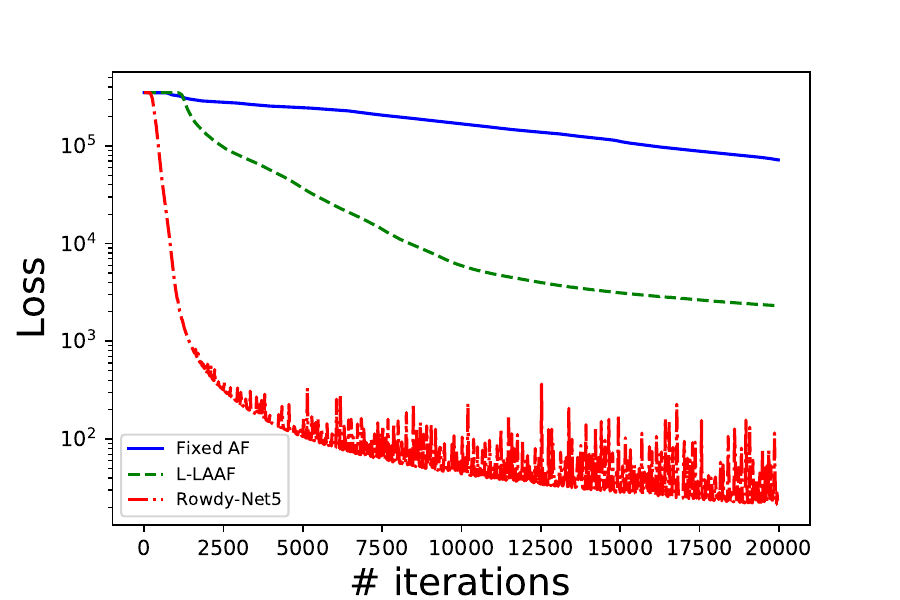}
\includegraphics[trim=1.2cm 0cm 0cm 0cm, scale=0.54, angle=0]{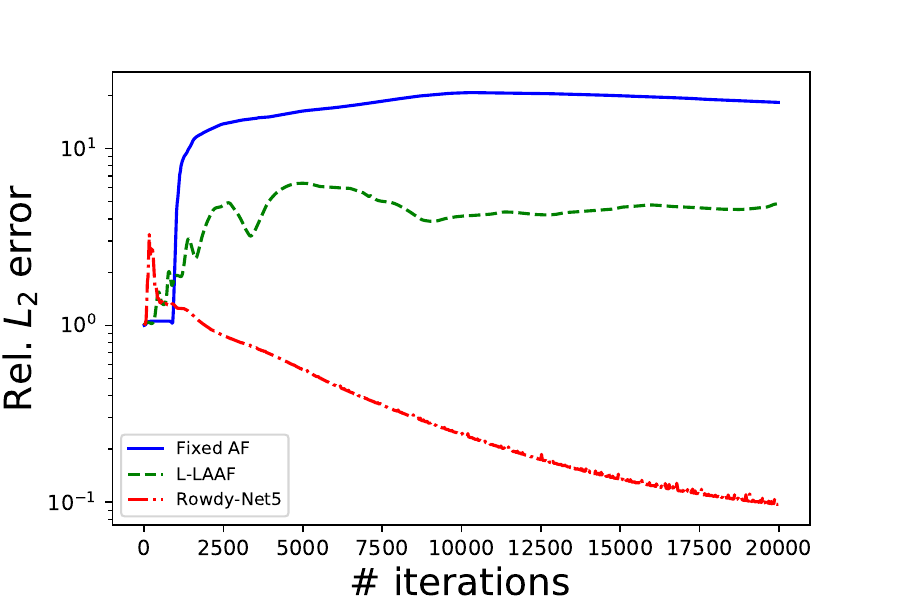}
\caption{High frequency solution of the Helmholtz equation: Loss function (left) and relative $L_2$ error (right) for the fixed activation, L-LAAF and Rowdy activation functions.}\label{fig:Test3_HF}
\end{figure}
\begin{figure} [htpb] 
\centering
\includegraphics[trim=1cm 0cm 0cm 0cm, scale=0.25, angle=0]{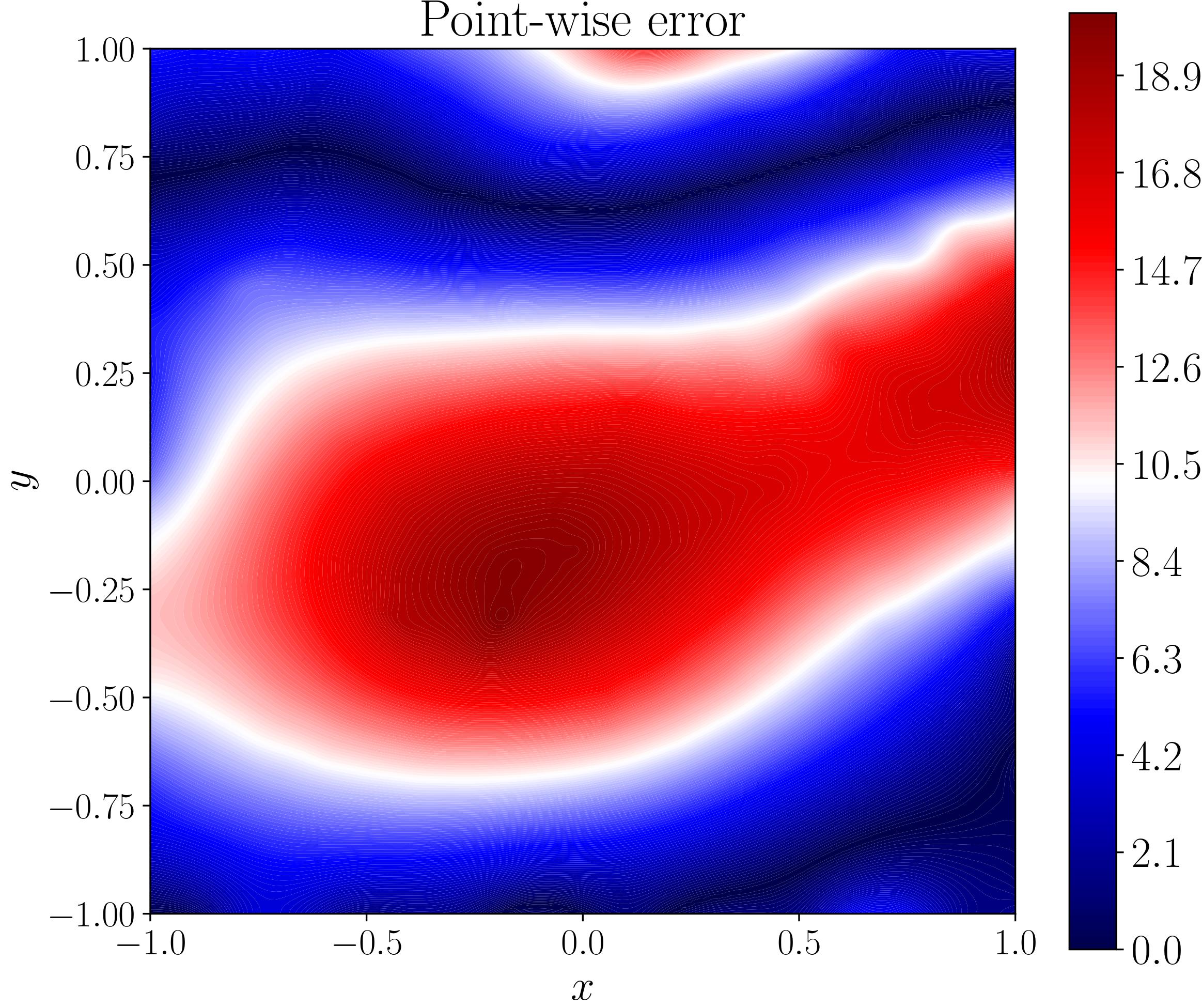}~~
\includegraphics[trim=1cm 0cm 0cm 0cm, scale=0.25, angle=0]{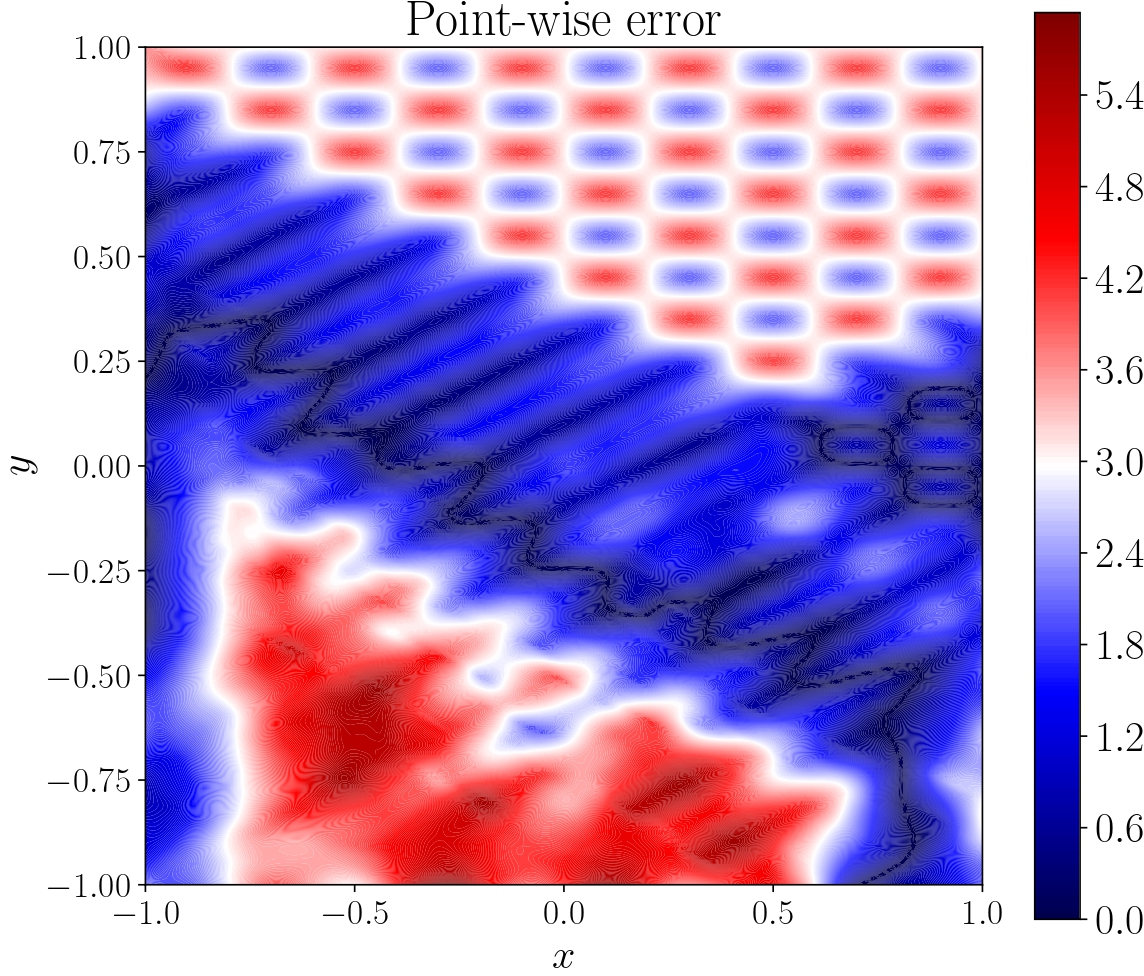}~~
\includegraphics[trim=1cm 0cm 0cm 0cm, scale=0.25, angle=0]{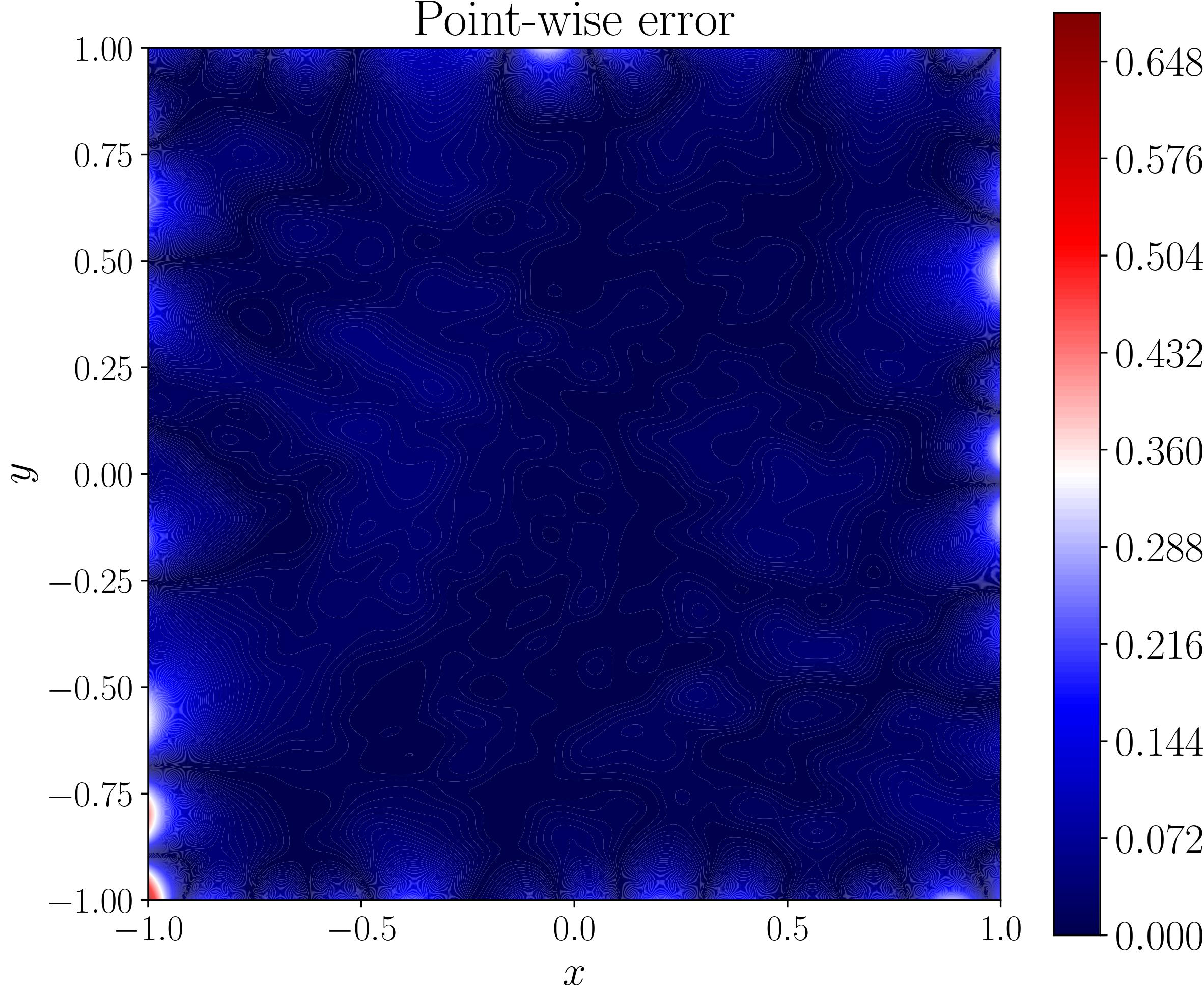}
\caption{High frequency solution of the Helmholtz equation: Point-wise absolute errors after 20k iterations for the fixed activation (left), L-LAAF (middle) and Rowdy activation functions (right).}\label{fig:Test3_HF2}
\end{figure}
Figure \ref{fig:Test3_HF} shows the loss function and relative $L_2$ error for the fixed activation, L-LAAF and Rowdy activation functions. In all cases we used 9e-5 learning rate, the activation function is the hyperbolic tangent, the number of residual points is 10k, and number of boundary data points is 400. The FNN consist of 3 hidden-layers with 60 neurons in each layer. Neither fixed nor L-LAAF converges even after 20k iterations for this problem, whereas the Rowdy-Net5 converges faster. The point-wise absolute error after 20k iterations is shown in figure \ref{fig:Test3_HF2} for the fixed activation, L-LAAF and Rowdy activation functions. The absolute error is large for both fixed and locally adaptive activation functions, whereas Rowdy-Net gives small error.

\subsection{Standard deep learning benchmark problems}
In the previous subsections, we have seen the advantages of the physics-informed Rowdy-Nets. One of the remaining questions is whether or not the advantages remain in the cases without physics information for other types of deep learning applications. This subsection presents numerical results with various standard benchmark problems in deep learning to explore the question.

\begin{figure}[b!]
\center
\begin{subfigure}[b]{0.32\textwidth}
  \includegraphics[width=\textwidth, height=\textwidth]{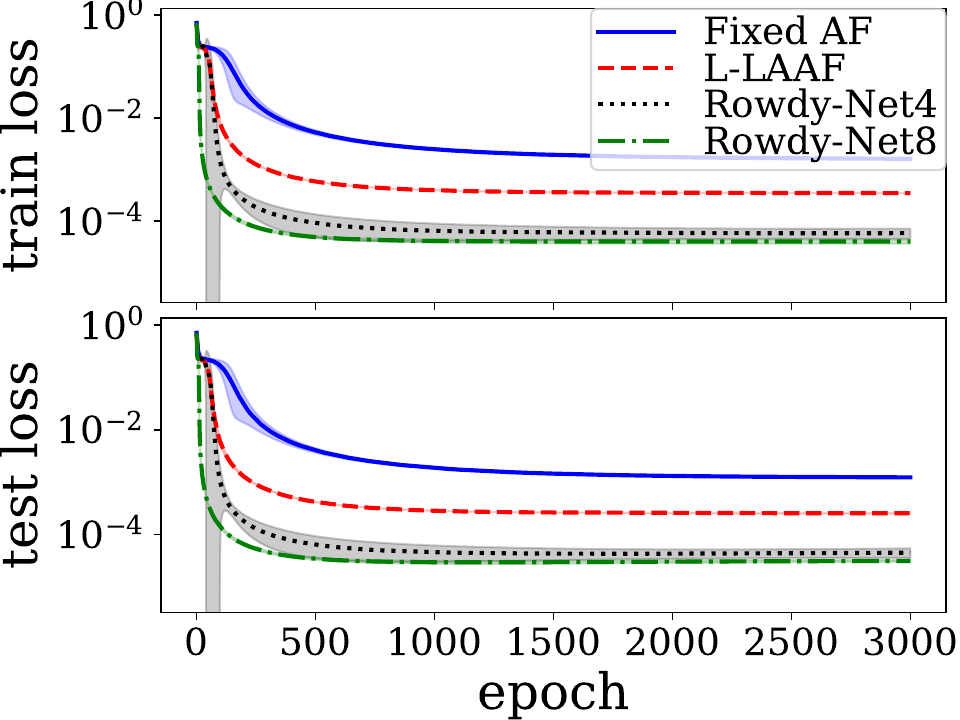}
  \caption{$n=1$}
\end{subfigure}
\begin{subfigure}[b]{0.32\textwidth}
  \includegraphics[width=\textwidth, height=\textwidth]{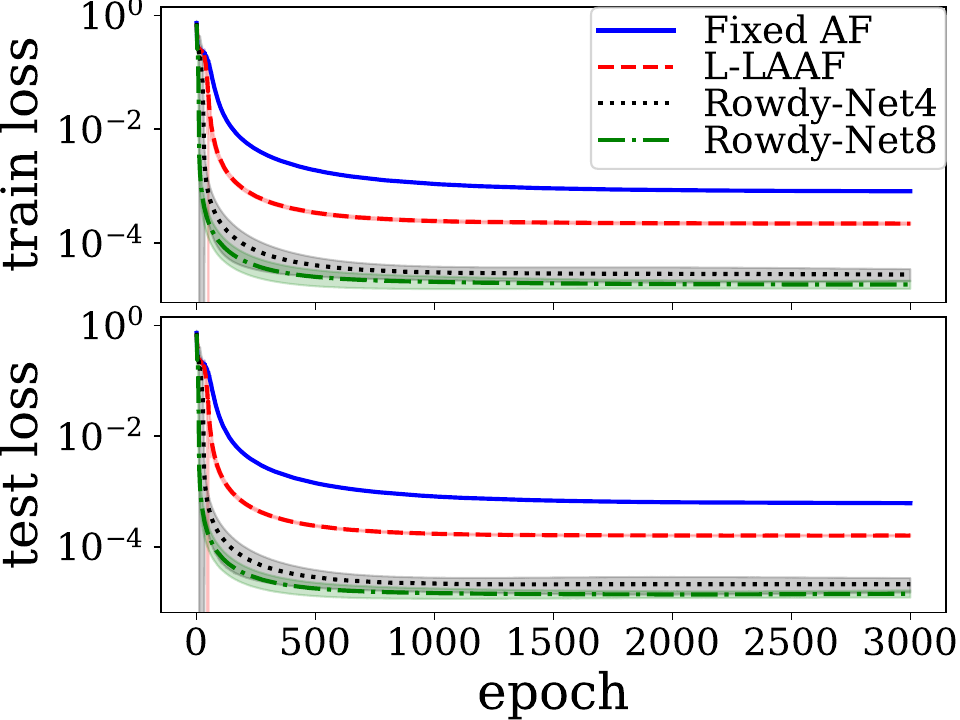}
  \caption{$n=2$}
\end{subfigure}
\begin{subfigure}[b]{0.32\textwidth}
  \includegraphics[width=\textwidth, height=\textwidth]{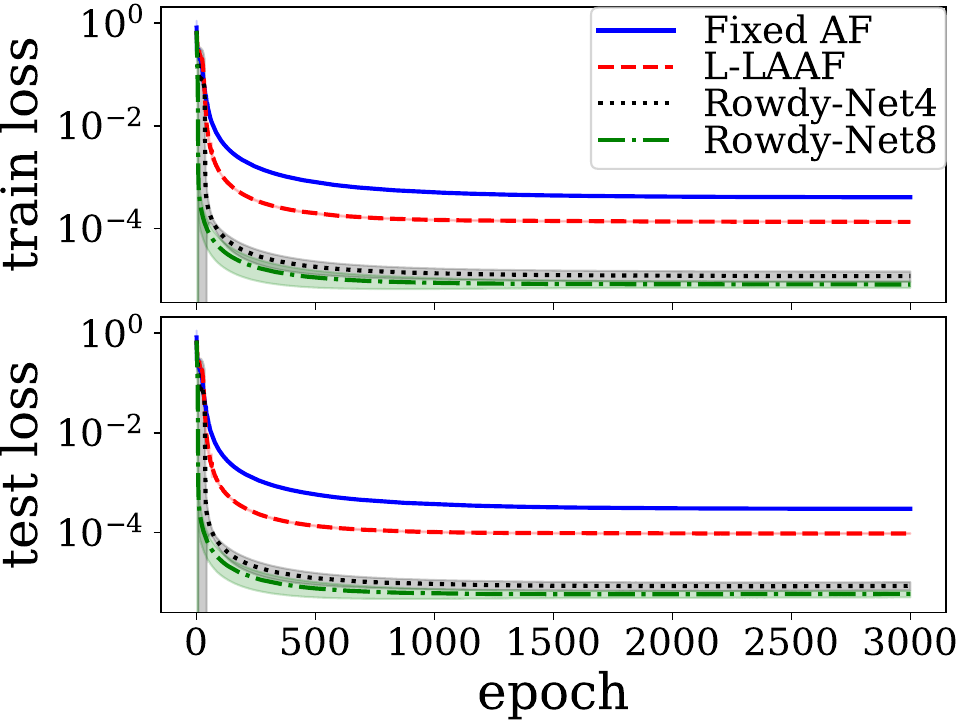}
  \caption{$n=4$}
\end{subfigure}
\caption{Fully-connected  neural networks for the two-moons dataset} 
\label{fig:new:1} 
\end{figure}

\begin{figure}[htpb!]
\center
\begin{subfigure}[b]{0.32\textwidth}
  \includegraphics[width=\textwidth, height=0.9\textwidth]{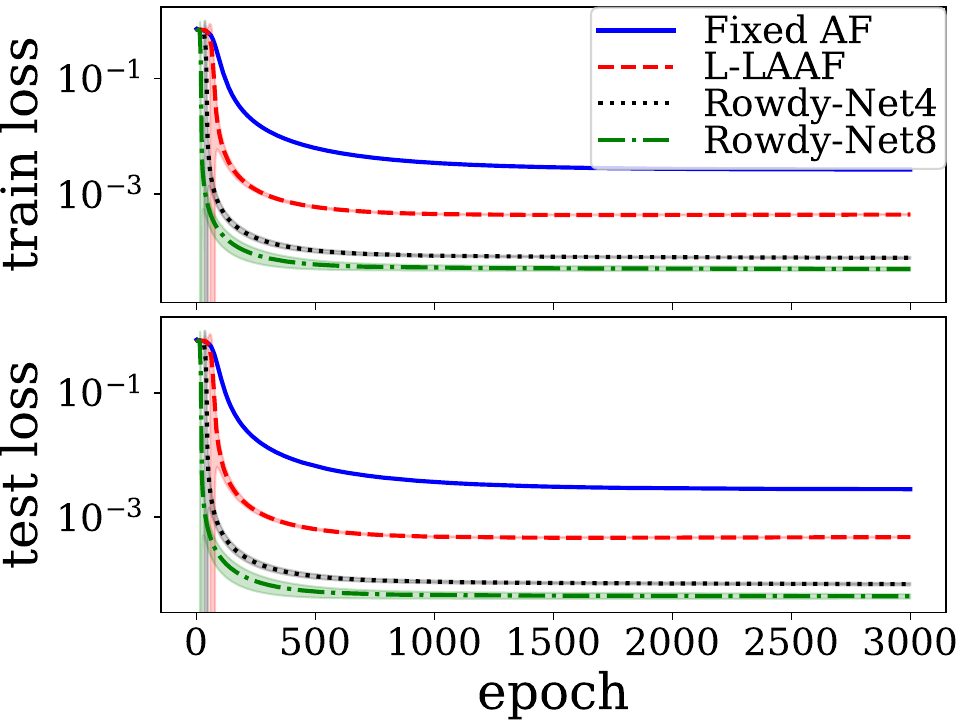}
  \caption{$n=1$}
\end{subfigure}
\begin{subfigure}[b]{0.32\textwidth}
  \includegraphics[width=\textwidth, height=0.9\textwidth]{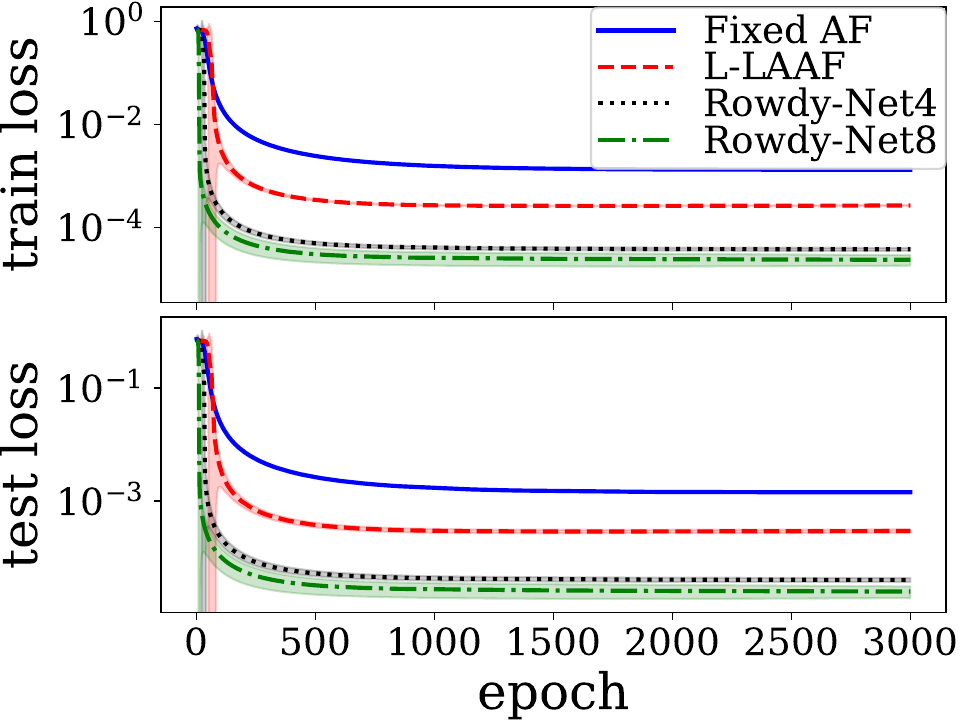}
  \caption{$n=2$}
\end{subfigure}
\begin{subfigure}[b]{0.32\textwidth}
  \includegraphics[width=\textwidth, height=0.9\textwidth]{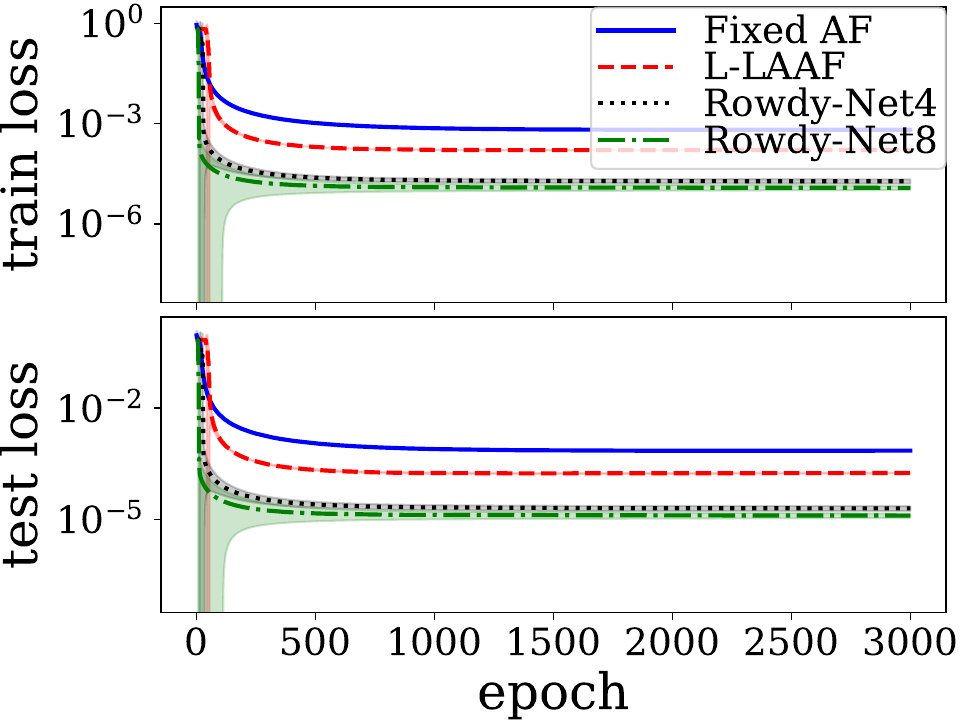}
  \caption{$n=4$}
\end{subfigure}
\caption{Fully-connected  neural networks for the two-circle dataset} 
\label{fig:new:2} 
\end{figure}

\begin{figure}[htpb!]
\center
\begin{subfigure}[b]{0.32\textwidth}
  \includegraphics[width=\textwidth, height=\textwidth]{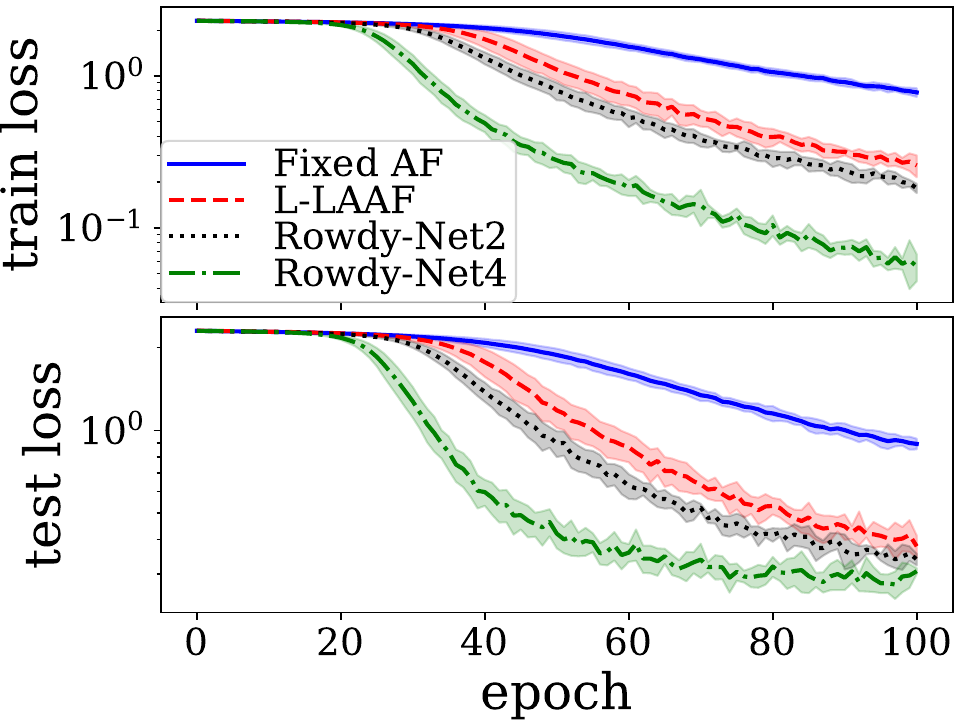}
  \caption{Semeion}
\end{subfigure}
\begin{subfigure}[b]{0.32\textwidth}
  \includegraphics[width=\textwidth, height=\textwidth]{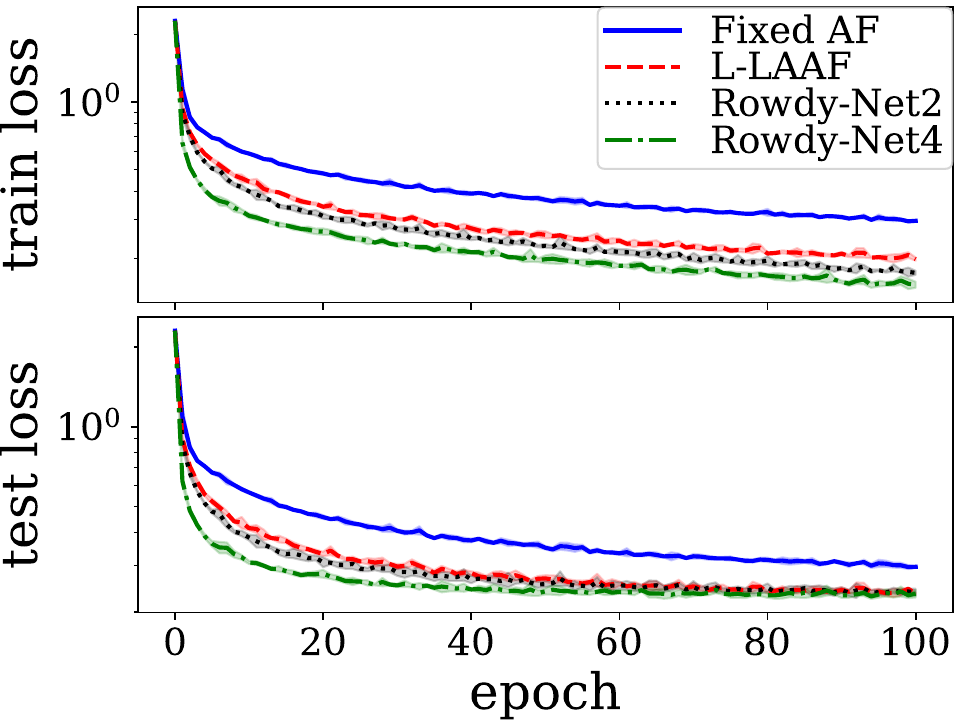}
  \caption{Fashion-MNIST}
\end{subfigure}
\begin{subfigure}[b]{0.32\textwidth}
  \includegraphics[width=\textwidth, height=\textwidth]{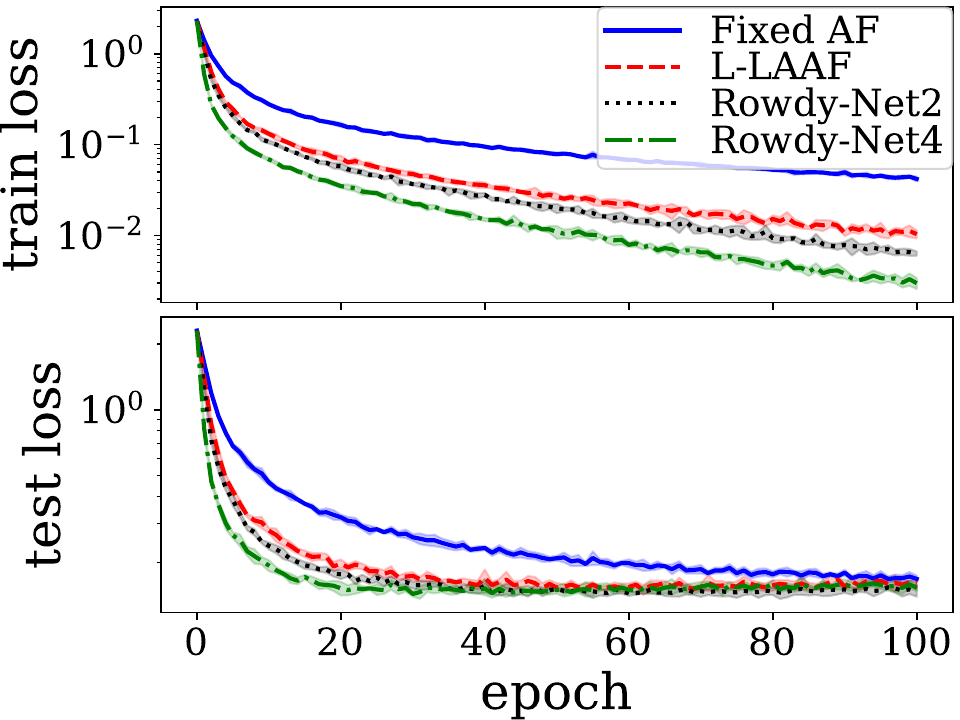}
  \caption{Kuzushiji-MNIST}
\end{subfigure}
\begin{subfigure}[b]{0.32\textwidth}
  \includegraphics[width=\textwidth, height=\textwidth]{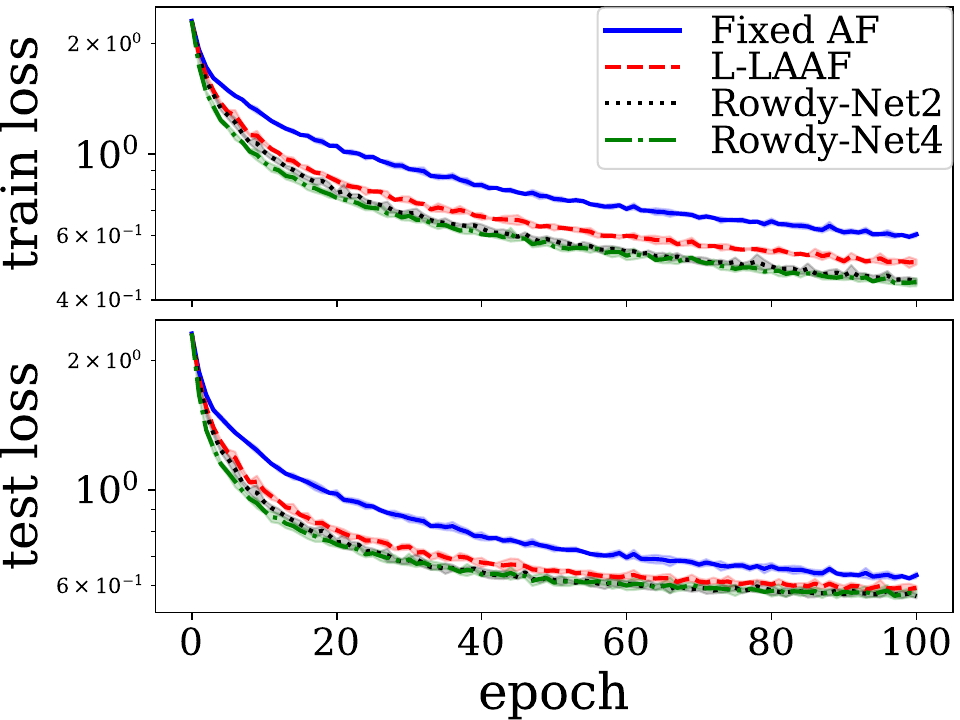}
  \caption{CIFAR-10}
\end{subfigure}
\begin{subfigure}[b]{0.32\textwidth}
  \includegraphics[width=\textwidth, height=\textwidth]{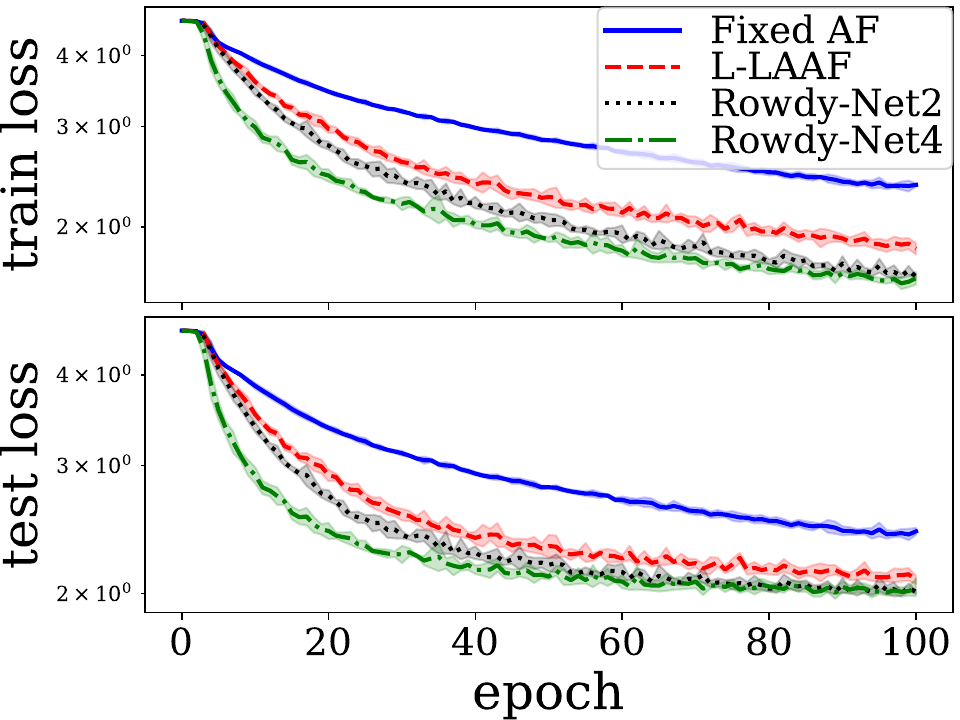}
  \caption{CIFAR-100}
\end{subfigure}
\caption{LeNet for various standard benchmark image datasets } 
\label{fig:new:3} 
\end{figure}

\begin{figure}[hbt]
\begin{minipage}{0.48\hsize}
\centering
\includegraphics[width=0.8\textwidth, height=0.6\textwidth]{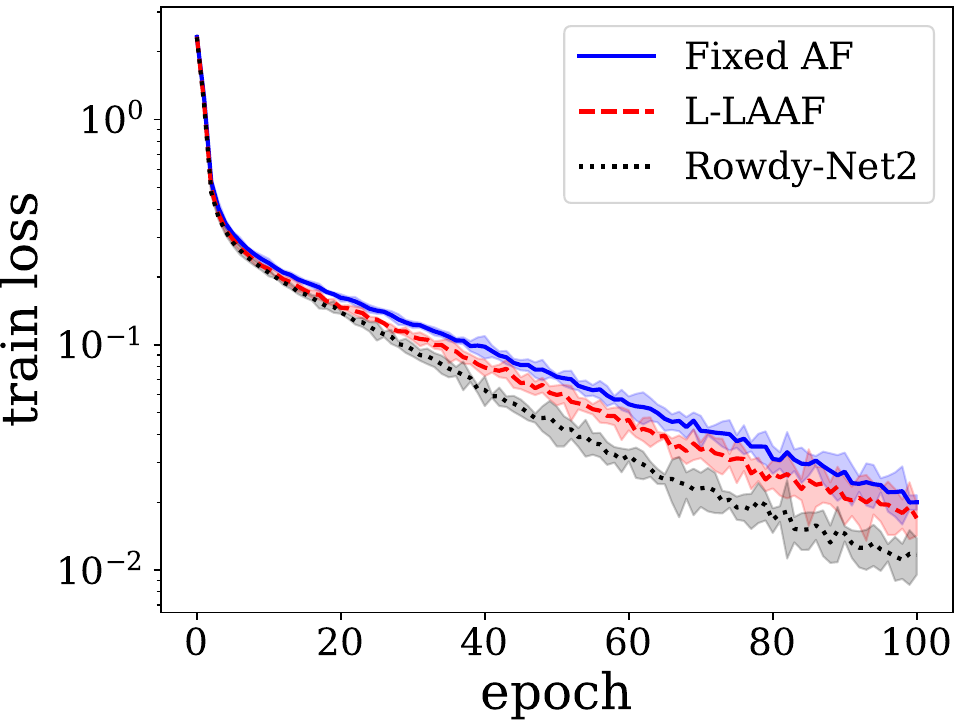}%
  \centering
  \captionof{figure}{ResNet for SVHN: training loss}
  \label{fig:new:4}
\end{minipage}
    \hfill
\begin{minipage}{0.48\hsize}
  \centering
  \begin{tabular}{cc} \hline
  Activation & Test error (\%) \\ \hline
  Fixed AF & 5.36 (0.13) \\
  L-LAAF  & 5.26 (0.20) \\
  Rowdy-Net2 & 4.92 (0.08) \\ \hline
  \end{tabular}
  \centering
\captionof{table}{ResNet for SVHN: test error}
     \label{table:new:1}
\end{minipage}
\end{figure}

\begin{figure}[b!]
\begin{minipage}{0.48\hsize}
\centering
\includegraphics[width=0.8\textwidth, height=0.6\textwidth]{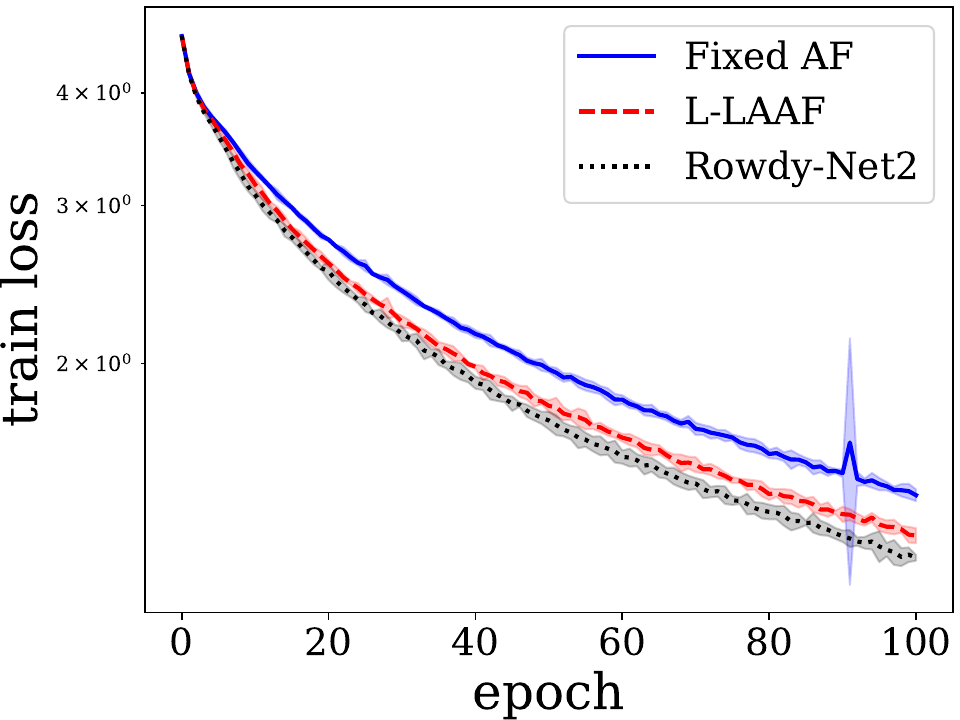}%
  \centering
  \captionof{figure}{ResNet for CIFAR-100: training loss}
  \label{fig:rev1:1}
\end{minipage}
    \hfill
\begin{minipage}{0.48\hsize}
  \centering
  \begin{tabular}{cc} \hline
  Activation & Test error (\%) \\ \hline
  Fixed AF & 35.80 (0.34)  \\
  L-LAAF  & 34.25 (0.29)   \\
  Rowdy-Net2 & 33.40 (0.27)   \\ \hline
  \end{tabular}
  \centering
\captionof{table}{ResNet for CIFAR-100: test error}
     \label{table:rev1:1}
\end{minipage}
\end{figure}

We first report the results for \textit{fully-connected} feed-forward neural networks. The results  are reported in Figures \ref{fig:new:1}--\ref{fig:new:2} with the mean  values and the uncertainty intervals. The   lines  are the mean values over three random trials and the shaded regions represent the intervals of $2 \times$(the sample standard deviations). As can be seen in the figures, the training and testing losses of Rowdy-Nets were lower than those of fixed AF and L-LAAF. In the figures, Rowdy-Net4 uses $K=4$, and Rowdy-Net8 uses $K=8$. For all the experiments, we used the network with three layers and  400 neurons per hidden layer and set the activation functions of all layers of fixed AF\ to be rectified linear unit (ReLU).  Each entry of the weight matrices was initialized independently by the normal distribution with the standard deviation being set to the reciprocal of the square root of   $400$ for all layers. We used the standard two-moons dataset and two-circles dataset with 1000 training data points and 1000 testing data points, generated by the scikit-learn command, \texttt{sklearn.datasets.make\_moons}   and \texttt{sklearn.datasets.make\_circles}    \cite{scikitlearn}. For each dataset, we used mini-batch stochastic gradient descent (SGD) with mini-batch size of 64 and the binary cross-entropy loss. We set the momentum coefficient to be 0.8,  the learning rate to be 0.001, and the weight-decay rate to be $10^{-4}$.

We now report the results for  \textit{convolutional}  deep neural networks. Figure \ref{fig:new:3} shows the results for the following standard variant of LeNet \cite{lecun1998gradient} with five layers: (1) input layer; (2) convolutional hidden layer  with $32  $ filters of size $5$-by-$5$ followed by max-pooling and activation functions; (3) convolutional hidden layer  with $32  $ filters of size $5$-by-$5$ followed by max-pooling and activation functions; (4)
fully-connected hidden layer with $256$ output units followed by activation functions; (5) fully-connected output layer.
Figure \ref{fig:new:4} and Table \ref{table:new:1} present the results for the standard pre-activation ResNet with $18$ layers \cite{he2016identity}. As can be seen in Figure \ref{fig:new:3}--\ref{fig:new:4} and Table \ref{table:new:1}, the Rowdy-Nets  outperformed fixed AF and L-LAAF in terms of training and testing performances for the convolutional networks as well. In the figures, Rowdy-Net2 uses $K=2$, and Rowdy-Net4 uses $K=4$.
In Figures \ref{fig:new:3}--\ref{fig:new:4}, the   lines  are the mean values over five random trials and the shaded regions represent the  intervals of $2 \times$(the sample standard deviations).  Table \ref{table:new:1} reports the mean test error and its standard deviation over five random trials for each method. Similarly, Figure \ref{fig:rev1:1} and Table \ref{table:rev1:1} report the result with the same setting (using the same ResNet) for a larger dataset, CIFAR-100. This additional result for a larger dataset shows qualitatively the same behaviors as those for smaller datasets.

For the convolutional networks, we adopted the standard benchmark datasets in deep learning   ---  Semeion \cite{srl1994semeion},  Fashion-MNIST \cite{xiao2017fashion}, Kuzushiji-MNIST  \cite{clanuwat2019deep}, CIFAR-10 and  CIFAR-100  \cite{krizhevsky2009learning},  and SVHN \cite{netzer2011reading}.  We used all the training and testing data points exactly as provided
by those datasets. For the Semeion dataset, since the default split of training and testing data points is not provided,  we randomly selected  1000 data points as training data points from the original 1593 data points;
the remaining 593 points were used as testing data points. For each dataset, we used the standard data-augmentation of images (e.g., random crop and random horizontal flip for CIFAR-10 and CIFAR-100). Semeion, Fashion-MNIST, and Kuzushiji-MNIST are the datasets with handwritten digits, images of clothing and accessories, and Japanese letters, respectively. CIFAR-10 is a popular dataset containing 50000 training and 10000 testing images in 10 classes with the image resolution of 32 $\times$ 32 with color. CIFAR-100 is a dataset similar to CIFAR-10, but it contains 600 images in each class for   100 classes. SVHN is a dataset containing images of street view house numbers obtained from  Google Street View images. 

As the qualitative behavior  was the same over different values of $n$ with ReLU networks in  Figures \ref{fig:new:1}--\ref{fig:new:2}, we fixed $n=1$ in the experiments for Figures \ref{fig:new:3}--\ref{fig:rev1:1} and Tables \ref{table:new:1}--\ref{table:rev1:1}. In Figures \ref{fig:new:3}--\ref{fig:rev1:1} and Tables \ref{table:new:1}--\ref{table:rev1:1}, all the model parameters  were initialized by the default initialization of PyTorch version 1.4.0 \cite{NEURIPS2019_9015}, which is based on the implementation in the previous work \cite{he2015delving}. Here, we used the cross-entropy loss. All other hyper-parameters for Figures \ref{fig:new:3}--\ref{fig:rev1:1} and Tables \ref{table:new:1}--\ref{table:rev1:1} were fixed to be the same values as the those in the experiments for Figures \ref{fig:new:1}--\ref{fig:new:2}. 

In summary, the experimental results in this subsection show that the Rowdy-Nets have the potential to improve the performance of standard activation functions without any prior physics information.

\section{Summary}
\label{sec:conclusions}
In this work we proposed a novel neural network architecture named as Kronecker neural networks (KNNs), which provides a general framework for neural networks with adaptive activation functions. Employing the Kronecker product in KNN makes the network wide while at the same time the number of trainable parameters remains low. 
For theoretical studies of the KNN, we analyzed its gradient flow dynamics and proved that at least in the beginning of gradient descent training, the loss by KNN is strictly smaller than those by feed-forward networks. 
Furthermore, we also established the global convergence of KNN in an over-fluctuating case.
In the same framework, we proposed a specific version, the Rowdy neural network (Rowdy-Net), which is a neural network with Rowdy activation functions. In the proposed Rowdy activation functions, noise in the form of sinusoidal fluctuations is injected in the activation function thereby removing the saturation zone from the output of every layer in the network, which allows the network to explore more and learn faster. The Rowdy activation functions easily capture the high-frequencies involved in the target function, hence they overcome the problem of spectral bias that is omnipresent in all neural networks as well as in PINNs. We also note that the Rowdy activations can be readily implemented in any neural network architecture. 

In the computational experiments we solved various problems such as function approximation using feed-forward neural networks and partial differential equation using physics-informed neural networks with standard (fixed) activation function, L-LAAF (layer-wise adaptive) and the proposed Rowdy activation functions. In all test cases, we obtained substantial improvement in the training speed as well as in the predictive accuracy of the solution. Moreover, the proposed Rowdy activation functions was shown to accelerate the minimization process of the loss values in various standard deep learning benchmark problems such as MNIST, CIFAR, SVHN etc., in agreement with the theoretical results.

\section*{Acknowledgments}
This work was supported by the U.S. Department of Energy PhILMs grant DE-SC0019453 and  OSD/AFOSR MURI Grant FA9550-20-1-0358.

\appendix
\section{Proof of Theorem~\ref{thm:small-loss}} \label{app:thm:small-loss}

\begin{proof}
    Let $\Theta$ be the set of parameters to be trained.
    Let
    \begin{equation*}
        \text{Res}(X) = [\text{Res}(x_1),\cdots, \text{Res}(x_m)]^T \in \mathbb{R}^m
    \end{equation*}
    where 
    $\text{Res}(x_j) = u(x_j;\Theta) - y_j$.
    Since 
    $L(\Theta) = \frac{1}{2}\|\text{Res}(X)\|^2$
    and 
    $\frac{\partial \text{Res}(x_j)}{\partial \Theta}=\frac{\partial u(x_j;\Theta)}{\partial \Theta}$, we have
    \begin{equation*}
        \nabla_{\Theta} L(\Theta) = \sum_{j=1}^m \text{Res}(x_j) \cdot 
        \frac{\partial u(x_j;\Theta)}{\partial \Theta}.
    \end{equation*}
    Note that 
    \begin{align*}
        \frac{d}{dt}L(t)
        &=  \frac{1}{2} \frac{d}{dt} \langle \text{Res}(X), \text{Res}(X) \rangle 
        =  \left\langle \text{Res}(X), \frac{d}{dt}\text{Res}(X) \right\rangle  =  \left\langle \text{Res}(X), \frac{d}{dt}\text{u}(X;\Theta(t)) \right\rangle.
    \end{align*}
    Let $\tilde{x} = [x, 1]^T$ and $v_i = [w_i; b_i]$. 
    Suppose $u(x;\Theta) = \sum_{i=1}^N c_i \left[\sum_{k=1}^K \alpha_k\phi_k(\omega_k\tilde{x}^Tv_i)\right]$,
    where $\Theta = \{c_i,v_i\}_{i=1}^N \cup \{\alpha_i,\omega_i\}_{i=1}^K$.
    Note that if the standard FF is considered,
    one can simply drop the terms related $\alpha$ and $\omega$.
    From the gradient-flow dynamics \eqref{def:GradFlow}, we have
    \begin{align*}
        \dot{c}_i(t) &= - \sum_{j=1}^m   \left[\sum_{k=1}^K \alpha_k\phi_k(\omega_k\tilde{x}_j^Tv_i)\right] \text{Res}(x_j) = - C_i \cdot Res(X), \\
        \dot{v}_i(t)
        &= - c_i \sum_{j=1}^m \tilde{x}_j  \left(\sum_{k=1}^K \alpha_k \omega_k \phi_k'(\omega_k \tilde{x}_j^Tv_i) \right) \text{Res}(x_j)  = -B_i \cdot Res(X), \\
        \dot{\alpha}_k &= - \sum_{j=1}^m \left(\sum_{i=1}^N c_i \phi_k(\omega_k \tilde{x}_j^Tv_i) \right) \text{Res}(x_j) = -A_k \cdot \text{Res}(X), \\
        \dot{\omega}_k &= - \sum_{j=1}^m \left(\sum_{i=1}^N c_i \alpha_k \tilde{x}_j^Tv_i  \phi_k'(\omega_k \tilde{x}_j^Tv_i) \right) \text{Res}(x_j) = -\Omega_k \cdot \text{Res}(X),
    \end{align*}
    for $1 \le i \le n$, $1\le k \le K$, and 
    \begin{equation} \label{def:M}
        \begin{split}
            C_i &= \begin{bmatrix}
        \sum_{k=1}^K \alpha_k\phi_k(\omega_k\tilde{x}_1^Tv_i), \cdots, \sum_{k=1}^K \alpha_k\phi_k(\omega_k\tilde{x}_m^Tv_i)
        \end{bmatrix} \in \mathbb{R}^m, \\
        B_i &= c_i\begin{bmatrix}
        \tilde{x}_1  \left(\sum_{k=1}^K \alpha_k \omega_k \phi_k'(\omega_k \tilde{x}_1^Tv_i) \right),
        & \cdots
        &
        \tilde{x}_m  \left(\sum_{k=1}^K \alpha_k \omega_k \phi_k'(\omega_k \tilde{x}_m^Tv_i) \right) 
        \end{bmatrix} \in \mathbb{R}^{(d+1)\times m}, \\
        A_k &= \begin{bmatrix}
        \sum_{i=1}^N c_i \phi_k(\omega_k \tilde{x}_1^Tv_i),
        & \cdots & \sum_{i=1}^N c_i \phi_k(\omega_k \tilde{x}_m^Tv_i)
        \end{bmatrix} \in \mathbb{R}^m, \\
        \Omega_k &=
        \alpha_k
        \begin{bmatrix}
        \sum_{i=1}^N c_i  \tilde{x}_1^Tv_i  \phi_k'(\omega_k \tilde{x}_1^Tv_i), & \cdots & 
        \sum_{i=1}^N c_i  \tilde{x}_m^Tv_i  \phi_k'(\omega_k \tilde{x}_m^Tv_i)
        \end{bmatrix} \in \mathbb{R}^m.
        \end{split}
    \end{equation}
    Let $\bm{C} = [C_1; \cdots; C_N] \in \mathbb{R}^{N\times m}$,
    $\bm{B} = [B_1; \cdots; B_N] \in \mathbb{R}^{(d+1)N \times m}$,
    $\bm{A} = [A_1; \cdots; A_K] \in \mathbb{R}^{K\times m}$
    and 
    $\bm{\Omega} = [\Omega_1; \cdots; \Omega_K] \in \mathbb{R}^{K\times m}$.
    By letting 
    $\bm{v} = (v_j) \in \mathbb{R}^{(d+1)N}$,
    $\alpha = (\alpha_j),
    \omega = (\omega_j)
    \in \mathbb{R}^K$,
    $\bm{C} = (c_j) \in \mathbb{R}^N$,
    and 
    $\Theta = \begin{bmatrix}
     \bm{c} \\ \bm{v} \\ \alpha \\ \omega
    \end{bmatrix} \in \mathbb{R}^{(d+2)n + 2K}$,
    we have
    \begin{align*}
        \dot{\bm{c}} = -\bm{C}\cdot \text{Res}(X), \qquad
        \dot{\bm{v}} = -\bm{B}\cdot \text{Res}(X), \qquad
        \dot{\bm{\alpha}} = -\bm{A}\cdot \text{Res}(X), \qquad
        \dot{\bm{\omega}} = -\bm{\Omega} \cdot \text{Res}(X).
    \end{align*}
    Let
    $\bm{M} = \begin{bmatrix}
        \bm{C}; \bm{B}; \bm{A}; \bm{\Omega}
        \end{bmatrix} \in \mathbb{R}^{((d+2)n + 2K) \times m}$.
    Since $\dot{\Theta} = -\bm{M}\cdot \text{Res}(X)$, 
    it can be checked that ,
    \begin{align*}
        \frac{d}{dt}\text{u}(X;\Theta(t))
        = \bm{M}^T\dot{\Theta}
        = - \bm{M}^T\bm{M} \cdot \text{Res}(X).
    \end{align*}
    Therefore,
    \begin{align*}
        \frac{d}{dt}L(t)
        &=  \left\langle \text{Res}(X), \frac{d}{dt}\text{u}(X;\Theta(t)) \right\rangle
        = - \|\bm{M}\cdot \text{Res}(X)\|^2
        \\
        &= - \|\begin{bmatrix}
        \bm{C} \\ \bm{B} \\ \bm{\Omega}
        \end{bmatrix} \cdot \text{Res}(X)\|^2
        - \|\bm{A} \cdot \text{Res}(X)\|^2.
    \end{align*}
    It follows from Lemma~\ref{lem:nonsingular-A} that 
    with probability 1 over initialization, we have
    $$
    \|\bm{A} \cdot \text{Res}(X)\|^2 \ge \sigma_{\min}^2(\bm{A}) \|\text{Res}(X)\|^2 > 0,
    $$
    which implies 
    \begin{align*}
        \frac{d}{dt}L^{\text{Rowdy}}(0) < \frac{d}{dt}L^{\text{FF}}(0) \le 0.
    \end{align*}
    Note that $\bm{\Psi}$ in Lemma~\ref{lem:nonsingular-A}
    is $\bm{A}^T$, $m \le K$ by Assumption~\ref{assumption:activation},
    and $\sigma_{\min}(\bm{A})$
    is the $m$-th largest singular value of $\bm{A}$.
    Since $\phi_k$'s are in $C^1$,
    it follows from the Peano existence theorem \cite{perko2013differential}
    that the gradient flow admits a solution $\Theta(t)$ in a neighborhood $I_0$ of $t=0$.
    Since $L^{\text{Rowdy}}(0) = L^{\text{FF}}(0)$,
    there exists $T > 0$ such that 
    for all $t \in (0, T)$,
    \begin{align*}
        L^{\text{Rowdy}}(t) <  L^{\text{FF}}(t),
    \end{align*}
    which shows that the Rowdy network induces a smaller training loss value
    in the beginning phase of the training.
\end{proof}

\begin{lemma} \label{lem:nonsingular-A}
    Suppose Assumptions~\ref{assumption:init} and ~\ref{assumption:activation} hold.
    For any non-degenerate data points $\{x_j\}_{j=1}^m$ where $m \le K$,
    with probability 1 over $\{c_i, v_i\}_{i=1}^N$,
    \begin{align*}
        [\bm{\Psi}]_{kj} = \psi_k(x_j), 
        \quad
        \text{where} \quad 
        \psi_k(x) = \sum_{i=1}^N c_i\phi_k(v_i^T\tilde{x}), \quad
        \tilde{x} = [x; 1], 
        \qquad 1\le k \le K, 1 \le j \le m,  
    \end{align*}
    is full rank.
    For the later use, let $\lambda_0$ be the $m$-th smallest singular value of $\bm{\Psi}$.
\end{lemma}
\begin{proof}
    Let $\bm{\Psi}_{k,:}$
    and $\bm{\Psi}_{:,j}$ be
    the $k$-th row and the $j$-th column of $\bm{\Psi}$, respectively.
    Suppose 
    \begin{align*}
        \sum_{k=1}^m \delta_k \bm{\Psi}_{k,:} = 0,
    \end{align*}
    for some $\delta = [\delta_1,\cdots, \delta_m]^T$.
    Then, for each $j=1,\cdots, m$, we have
    \begin{align*}
        0 = \sum_{k=1}^m \delta_k \psi_k(x_i) 
        =  \sum_{k=1}^m \delta_k 
        c^T \Phi_k(x_i)
        = c^T\left(\sum_{k=1}^m \delta_k \Phi_k(x_i) \right), 
        \quad \text{where} \quad
        \Phi_k(x_i) = \begin{bmatrix}
        \phi_k(v_1^T\tilde{x}_i) \\
        \vdots \\
        \phi_k(v_N^T\tilde{x}_i)
        \end{bmatrix}.
    \end{align*}
    Hence, with probability 1 over the initialization of $c$, 
    we have
    \begin{align*}
        \sum_{k=1}^m \delta_k \Phi_k(x_i) = 0
        \iff
        \sum_{k=1}^m \delta_k \phi_k(v_s^T\tilde{x}_i) = 0 \quad \forall 1\le s \le n, 1 \le i \le m.
    \end{align*}
    With probability 1 over $v_s$, 
    $v_s^Tx_1,\cdots, v_s^Tx_m$ are distinct.
    It then follows from Assumption~\ref{assumption:activation}
    that for each $s = 1,\cdots, n$,
    $$
    [\bm{\Phi}_s]_{ik} = \phi_k(v_s^T\tilde{x}_i), \qquad 1 \le i \le m, 1 \le k \le m,
    $$ 
    is full rank.
    Therefore, we conclude that $\delta_1 = \cdots =\delta_K = 0$,
    which implies that $\bm{\Psi}$ is also full rank.
\end{proof}

\section{Proof of Theorem~\ref{thm:convergence}} \label{app:thm:convergence}

\begin{proof}
    It follows from the proof of Theorem~\ref{thm:small-loss}
    that  
    \begin{align*}
        \|\frac{d}{dt} \Theta(t) \|^2 = - \frac{d}{dt} L(t) \implies
        \|\Theta(t) - \Theta(0) \| \le \sqrt{L(0)},
    \end{align*}
    for all $t \ge 0$.
    By Lemma~\ref{lemma:misfit-init}, 
    with probability at least $1- e^{-\frac{m\delta^2}{2\|X\|^2}}$, 
    \begin{align*}
        \sqrt{2L(0)} \le \|\bm{y}\|\left(1 + (1+\delta)B/K \right).
    \end{align*}
    Note that 
    \begin{align*}
        &\frac{\|\bm{y}\|\left(1 + (1+\delta)B/K \right)}{\sqrt{2}} \le \frac{\lambda_0}{2\sqrt{1 + (\max_j \|\tilde{x}_j\|\cdot \|\alpha\|_{\infty})^2} \cdot \|c\|_1\cdot B\sqrt{Km}}  \\
        \iff&
        \|c\|_1 \le \frac{\lambda_0}{\sqrt{2}\|\bm{y}\|\left(1 + (1+\delta)B/K \right) \sqrt{1 + (\max_j \|\tilde{x}_j\|\cdot \|\alpha\|_{\infty})^2} \cdot B\sqrt{Km}}  \\
        \iff&
        \frac{\|\bm{y}\|}{K\sqrt{m}}
        \le \frac{\lambda_0}{\sqrt{2}\|\bm{y}\|\left(1 + (1+\delta)B/K \right) \sqrt{1 + (\max_j \|\tilde{x}_j\|\cdot \|\alpha\|_{\infty})^2} \cdot B\sqrt{Km}} 
        \\
        \iff&
        \frac{1}{\sqrt{K}}
        \le \frac{\lambda_0}{\sqrt{2}\|\bm{y}\|^2\left(1 + (1+\delta)B/K \right) \sqrt{1 + (\max_j \|\tilde{x}_j\|\cdot \|\alpha\|_{\infty})^2} \cdot B}.
    \end{align*}
    Therefore,
    if
    \begin{align*}
        \frac{1}{\sqrt{K}}
        \le \frac{\lambda_0}{\sqrt{2}\|\bm{y}\|^2\left(1 + (1+\delta)B/K \right) \sqrt{1 + (\max_j \|\tilde{x}_j\|\cdot \|\alpha\|_{\infty})^2} \cdot B},
    \end{align*}
    with probability at least $1- e^{-\frac{m\delta^2}{2\|X\|^2}}$, we have
    \begin{align*}
        \|\Theta(t) - \Theta(0) \| \le \frac{\lambda_0}{2\sqrt{1 + (\max_j \|\tilde{x}_j\|\cdot \|\alpha\|_{\infty})^2} \cdot \|c\|_1\cdot B
        \sqrt{Km}}
    \end{align*}
    and it  follows from Lemma~\ref{lemma:nonsingular-A-t} that 
    $\sigma_{\min}(\bm{A}(t)) \ge \frac{\lambda_0}{2}$
    for all $t \ge 0$.
    Therefore,
    \begin{align*}
        L(t) \le L(0)e^{-\frac{\lambda_0}{2} t}, \qquad \forall t \ge 0.
    \end{align*}
    Since $\max_j \|\tilde{x}_j\| = 1$
    and $\|\alpha\|_{\infty} \le 1$,
    the proof is completed.
\end{proof}

\begin{lemma} \label{lem:perturbation}
    Suppose Assumptions~\ref{assumption:init} and ~\ref{assumption:activation}
    hold.
    Let $\Theta = \{\alpha_i\}_{i=1}^K \cup \{v_i\}_{i=1}^N$
    and $\tilde{\Theta} = \{\tilde{\alpha}_i\}_{i=1}^K \cup \{\tilde{v}_i\}_{i=1}^N$.
    Suppose $\phi_k(x)$'s are bounded by $B$ in $\mathbb{R}$.
    Then,
    \begin{align*}
        \|\bm{\Psi}(\Theta) - \bm{\Psi}(\Theta(0))\| \le
        \sqrt{1 + (\max_j \|\tilde{x}_j\|\cdot \|\alpha(0)\|_{\infty})^2} \cdot \|c\|_1\cdot B
        \sqrt{Km} \cdot \|\Theta - \Theta(0)\|_F,
    \end{align*}
    where $\bm{\Psi}$ is defined in Lemma~\ref{lem:nonsingular-A}.
\end{lemma}
\begin{proof}
    Let us denote its corresponding networks by
    \begin{align*}
        u(x) = \sum_{i=1}^N c_i \sum_{k=1}^K \alpha_k \phi_k(v_i^T\tilde{x}),
        \qquad
        \tilde{u}(x) = \sum_{i=1}^N c_i \sum_{k=1}^K \tilde{\alpha}_k \phi_k(\tilde{v}_i^T\tilde{x}).
    \end{align*}
    Note that 
    \begin{align*}
        |\psi_k(x) - \tilde{\psi}_k(x)|
        &= \left| \sum_{i=1}^N c_i (\alpha_k \phi_k(v_i^T\tilde{x}) - \tilde{\alpha}_k\phi_k(\tilde{v}_i^T\tilde{x})) \right| \\
        &\le
        \left| \sum_{i=1}^N c_i (\tilde{\alpha}_k - \alpha_k  )\phi_k(\tilde{v}_i^T\tilde{x}) \right|
        +
        \left| \sum_{i=1}^N c_i \alpha_k 
        (\phi_k(v_i^T\tilde{x}) -\phi_k(\tilde{v}_i^T\tilde{x})) \right|
        \\
        &
        \le 
        B \|c\|_1 \cdot |\alpha_k - \tilde{\alpha}_k|
        +
        B\|\tilde{x}\|\cdot |\alpha_k|\sum_{i=1}^N |c_i|\cdot \|v_i - \tilde{v}_i\| 
        \\
        &\le B \|c\|_1 \cdot |\alpha_k - \tilde{\alpha}_k|
        + B\|\tilde{x}\|\cdot |\alpha_k|\cdot  \|c\|_2 \cdot \|V - \tilde{V}\|_F \\
        &\le B\sqrt{\|c\|_1^2 + (\|\tilde{x}\|\cdot |\alpha_k|\cdot  \|c\|_2)^2}
        \sqrt{|\alpha_k - \tilde{\alpha}_k|^2 +\|V - \tilde{V}\|_F^2},
    \end{align*}
    where 
    $$
    V = \begin{bmatrix}
    v_1, & \cdots, & v_N
    \end{bmatrix}, 
    \qquad
    \tilde{V} = \begin{bmatrix}
    \tilde{v}_1, & \cdots, & \tilde{v}_N 
    \end{bmatrix}.
    $$
    Therefore,
    \begin{align*}
        \|\bm{\Psi} - \tilde{\bm{\Psi}}\| \le \|\bm{\Psi} - \tilde{\bm{\Psi}}\|_F = \sqrt{1 + (\max_j \|\tilde{x}_j\|\cdot \|\alpha\|_{\infty})^2} \cdot \|c\|_1\cdot B
        \sqrt{Km} \cdot \|\Theta - \Theta(0)\|_F.
    \end{align*}
\end{proof}

\begin{lemma} \label{lemma:nonsingular-A-t}
    Let $\Theta(0) = \{\alpha_i\}_{i=1}^K \cup \{v_i\}_{i=1}^N$
    and $\Theta(t) = \{\tilde{\alpha}_i\}_{i=1}^K \cup \{\tilde{v}_i\}_{i=1}^N$.
    Let $\sigma_{\min}(\bm{A}(0)) = \lambda_0$.
    Suppose 
    $$
    \|\Theta(t) - \Theta(0)\|_F \le \frac{\lambda_0}{2\sqrt{1 + (\max_j \|\tilde{x}_j\|\cdot \|\alpha\|_{\infty})^2} \cdot \|c\|_1\cdot B
        \sqrt{Km}},
    $$
    for all $0 \le t \le T$.
    Then
    $\sigma_{\min}(\bm{A}(t)) \ge \frac{\lambda_0}{2}$
    for all $0 \le t \le T$.
\end{lemma}
\begin{proof}
    Observe that 
    \begin{align*}
        \sigma_{\min}(A(t)) \ge \sigma_{\min}(A(0)) - \|A(t) - A(0)\|
        \ge \frac{\lambda_0}{2},
    \end{align*}
    where the second inequality follows from Lemma~\ref{lem:perturbation}.
\end{proof}

\begin{lemma} \label{lemma:misfit-init}
    Suppose $\alpha$ is initialized to satisfy $K\|\alpha\|_{\infty} \le 1$
    and 
    \begin{align*}
        \sum_{k=1}^K \alpha_k \mathbb{E}_{z\sim N(0,1)}[\phi_k(z)] = 0.
    \end{align*}
    Then, with probability at least $1- e^{-\frac{m\delta^2}{2\|X\|^2}}$
    over $\{v_i\}_{i=1}^N$,
    \begin{align*}
        \sqrt{2L(0)} \le \|\bm{y}\|\left(1 + (1+\delta)B \right).
    \end{align*}
\end{lemma}
\begin{proof}
    Let
    \begin{align*}
        [\bm{\Phi}]_{ij} = \sum_{k=1}^K \alpha_k \phi_k(v_i^T\tilde{x}_j), \qquad
        1\le i \le n, 1 \le j \le m.
    \end{align*}
    Note that 
    \begin{align*}
        \left|[\bm{\Phi}-\tilde{\bm{\Phi}}]_{ij}\right|
        &=  \left|\sum_{k=1}^K \alpha_k\phi_k(v_i^T\tilde{x}_j)
        -  \sum_{k=1}^K \alpha_k\phi_k(\tilde{v}_i^T\tilde{x}_j) \right|
        \\
        &\le 
        \sum_{k=1}^K |\alpha_k|\cdot |\phi_k(v_i^T\tilde{x}_j) -\phi_k(\tilde{v}_i^T\tilde{x}_j)|
        \le KB \|\alpha\|_{\infty} \cdot \|\tilde{x}_j^T(v_i - \tilde{v}_i)\|.
    \end{align*}
    Therefore,
    $$
    \|\bm{\Phi}-\tilde{\bm{\Phi}}\| \le 
    \|\bm{\Phi}-\tilde{\bm{\Phi}}\|_F \le K\|\alpha\|_{\infty} B\cdot \|\bm{X}\| \cdot \|V - \tilde{V}\|_F,
    $$
    which implies that
    \begin{align*}
        \left| \|\bm{\Phi}c\| - \|\tilde{\bm{\Phi}}c \| \right|
        \le
        \|\bm{\Phi}-\tilde{\bm{\Phi}}\|\cdot \|c\|
        \le K\|\alpha\|_{\infty} B \|c\|\cdot \|\bm{X}\| \cdot \|V - \tilde{V}\|_F.
    \end{align*}
    Thus, $\|\bm{\Phi}c\|$ is a Lipschitz function of $V$
    whose Lipschtiz constant is $K\|\alpha\|_{\infty} B \|c\|\cdot \|\bm{X}\|$.
    
    Also, since $\|\tilde{x}_j\| = 1$ for all $1\le j \le m$, we have 
    \begin{align*}
        \mathbb{E}_{V}[\|\bm{\Phi}c\| ]
        &\le \sqrt{\mathbb{E}_{V}[\|\bm{\Phi}c\|^2]} \\
        &= \sqrt{\sum_{j=1}^m \mathbb{E}_{V}\left[\left(\sum_{i=1}^N c_i \sum_{k=1}^K \alpha_k \phi_k(v_i^T\tilde{x}_j) \right)^2\right]} 
        \\
        &=\sqrt{m} \sqrt{\mathbb{E}_{g \sim N(0,I_N)}\left[\left(\sum_{i=1}^N c_i \sum_{k=1}^K \alpha_k \phi_k(g_i) \right)^2\right] }
        \\
        &= \sqrt{m}  \sqrt{
        \|c\|^2\mathbb{E}_{g \sim N(0,1)}\left[\left(\sum_{k=1}^K \alpha_k (\phi_k(g) - \mathbb{E}[\phi_k(g)]) \right)^2 \right]
        } 
         \\
        &\le \sqrt{m}  K \|\alpha\|_{\infty} B\|c\|.
    \end{align*}
    
    We recall (e.g. \cite{vershynin2018high}) that since $\|\bm{\Phi}c\|$ is a Lipschitz function of $V$,
    for $V \sim N(0,I_{d+1})$,
    with probability at least $1- e^{-\frac{t^2}{2(K\|\alpha\|_{\infty}B\|c\|\cdot \|X\|)^2}}$, 
    \begin{align*}
        \|\bm{\Phi}c\| \le \mathbb{E}_{V}[\|\bm{\Phi}c\|] + t.
    \end{align*}
    By letting $t = \delta \sqrt{m}  K \|\alpha\|_{\infty} B\|c\|$,
    we conclude that 
    with probability at least $1- e^{-\frac{m\delta^2}{2\|X\|^2}}$, 
    \begin{align*}
        \|\bm{\Phi}c\| \le  (1+\delta) \sqrt{m}  K \|\alpha\|_{\infty} B\|c\|.
    \end{align*}
    Since $|c_i| = \frac{\|y\|}{K\sqrt{mn}}$ and $K\|\alpha\|_{\infty} \le 1$, with probability at least $1- e^{-\frac{m\delta^2}{2\|X\|^2}}$, 
    we have
    \begin{align*}
        \|\bm{\Phi}c - \bm{y}\| \le \|\bm{\Phi}c\| + \|\bm{y}\|
        \le \|\bm{y}\|\left(1 + (1+\delta)B/K \right).
    \end{align*}
\end{proof}

\bibliography{references}

\end{document}